\documentclass[11pt]{article}

\usepackage{arxiv-style}
\usepackage{xurl}
\usepackage{amsmath}
\usepackage{amsmath}
\usepackage{amsfonts}
\usepackage{amssymb}
\usepackage{amsthm}
\usepackage{algorithm}
\usepackage{algpseudocode}
\usepackage{bm, bbm}
\usepackage{dsfont}
\usepackage{booktabs}           %
\usepackage{pifont}     %
\usepackage{multirow, multicol}
\usepackage{svg}
\usepackage{xcolor} 
\usepackage{colortbl}
\usepackage{subcaption}
\usepackage{centernot}
\usepackage{lipsum}
\usepackage{extpfeil} %

\definecolor{darkblue}{RGB}{0,0,139}
\definecolor{red}{RGB}{210,35,35}
\usepackage[
  colorlinks=true,
  linkcolor=darkblue,
  citecolor=darkblue,
  urlcolor=darkblue,
  filecolor=darkblue
]{hyperref}

\usepackage{aliascnt}

\newtheorem{theorem}{Theorem}

\newaliascnt{lemma}{theorem}
\newtheorem{lemma}[lemma]{Lemma}
\aliascntresetthe{lemma}

\newaliascnt{proposition}{theorem}

\aliascntresetthe{proposition}

\newaliascnt{corollary}{theorem}

\aliascntresetthe{corollary}

\newaliascnt{remark}{theorem}

\aliascntresetthe{remark}

\theoremstyle{definition}

\newaliascnt{definition}{theorem}

\aliascntresetthe{definition}

\newaliascnt{example}{theorem}

\aliascntresetthe{example}

\newtheorem{assumption}{Assumption}

\usepackage{cleveref}

\crefname{theorem}{theorem}{theorems}
\Crefname{theorem}{Theorem}{Theorems}

\crefname{lemma}{lemma}{lemmas}
\Crefname{lemma}{Lemma}{Lemmas}

\crefname{proposition}{proposition}{propositions}
\Crefname{proposition}{Proposition}{Propositions}

\crefname{corollary}{corollary}{corollaries}
\Crefname{corollary}{Corollary}{Corollaries}

\crefname{definition}{definition}{definitions}
\Crefname{definition}{Definition}{Definitions}

\crefname{assumption}{assumption}{assumptions}
\Crefname{assumption}{Assumption}{Assumptions}

\crefname{appendix}{appendix}{appendices}
\Crefname{appendix}{Appendix}{Appendices}

\newcommand{\yes}{\ding{51}}
\newcommand{\no}{\ding{55}}

\newcommand{\Ind}{\mathds{1}}

\DeclareMathOperator{\Regret}{Regret}
\DeclareMathOperator{\Violation}{Violation}

\DeclareMathOperator{\spn}{sp}

\DeclareMathOperator*{\argmax}{arg\,max}
\DeclareMathOperator*{\argmin}{arg\,min}

\newcommand{\bphi}{\bm{\phi}}
\newcommand{\bmu}{\bm{\mu}}
\newcommand{\btheta}{\bm{\theta}}

\newcommand{\qstq}{\quad\mathrm{s.t.}\quad}

\newcommand{\bigO}{\mathcal{O}}
\newcommand{\tbigO}{\widetilde{\mathcal{O}}}
\newcommand{\RegAug}{\Regret_{\text{aug}}}

\newcommand{\calA}{{\mathcal{A}}}

\newcommand{\calF}{{\mathcal{F}}}

\newcommand{\calM}{{\mathcal{M}}}

\newcommand{\calS}{{\mathcal{S}}}

\newcommand{\calZ}{{\mathcal{Z}}}

\newcommand{\bbE}{\mathbb{E}}

\newcommand{\bbR}{\mathbb{R}}

\newcommand{\bbZ}{\mathbb{Z}}

\title{Learning Infinite-Horizon Average-Reward\\CMDPs via State Augmentation}
\author{
Kihyun Yu$^{1,\star}$ \quad Seoungbin Bae$^{1,\star}$ \quad Dabeen Lee$^{2}$\\[0.3em]
$^{1}$ KAIST \quad 
$^{2}$ Seoul National University\\
\texttt{\{khyu99, sbbae31\}@kaist.ac.kr, dabeenl@snu.ac.kr}
}

\date{}

\begin{document}
\maketitle

\begin{abstract}
We study infinite-horizon average-reward constrained Markov decision processes (CMDPs) under the weakly communicating assumption. Existing high-probability guarantees for this setting either require computationally inefficient algorithms or have suboptimal dependence on the number of interactions $T$. We propose, to the best of our knowledge, the first computationally efficient algorithm that achieves $\widetilde{\mathcal{O}}(\sqrt{T})$ regret and cumulative constraint violation with high probability in the tabular setting. The $\sqrt{T}$ dependence is optimal up to logarithmic factors. Our approach incorporates cumulative constraint violation into the state and defines a reshaped reward through differences of a Huber potential. The added state determines the penalty on further violations while the reward function remains fixed on the augmented state space. Since the added state has known deterministic dynamics, only the original transition kernel needs to be estimated. The bounded slope of the Huber potential keeps the per-step reward bounded, and the potential differences telescope to relate the reshaped return to the original cumulative reward and the terminal potential. These properties allow us to apply finite-horizon approximation and optimistic value iteration with clipping, as used in unconstrained average-reward MDPs, without worsening the regret rate in $T$.
\end{abstract}

\def\thefootnote{$\star$}\footnotetext{These authors contributed equally to this work}\def\thefootnote{\arabic{footnote}}
\section{Introduction}
\label{sec:introduction}

In reinforcement learning (RL), actions that improve task performance may also increase resource consumption or the risk of undesirable outcomes~\citep{garcia2015comprehensive}. For example, a robotic controller must account for safety costs while learning to perform its task~\citep{robotics-achiam2017constrained}. Constrained Markov decision processes (CMDPs) model these requirements by maximizing expected reward subject to prescribed bounds on expected costs~\citep{altman2021constrained}. When the transition dynamics are unknown, the learner must acquire information about the effects of its actions on both reward and constraint satisfaction. Its performance is therefore measured by regret relative to an optimal feasible policy and by the constraint violation accumulated during learning.

Learning in CMDPs has been studied extensively in the finite-horizon~\citep{efroni2020exploration,qiu2020upper,liu2021learning,muller2024truly} and discounted settings~\citep{chen2021primal,xu2021crpo,ding2023last}. The infinite-horizon average-reward setting remains less understood, although it directly describes continuing tasks whose performance is evaluated over long periods~\citep{mahadevan1996average,sutton2018reinforcement}. In this setting, the agent interacts with the environment without resets, and both reward and constraint costs are evaluated through their long-run averages. For example, wireless transmission control can be formulated as minimizing average power consumption subject to a bound on average delay~\citep{djonin2007qlearning}. The average-reward formulation captures these requirements without choosing a terminal horizon or a discount factor. This distinction also matters for feasibility: a bound on discounted cost need not enforce the corresponding bound on long-run average cost~\citep{agnihotri2024acpo}.

We consider weakly communicating CMDPs, which allow stationary policies with multiple recurrent classes and are more general than ergodic and unichain CMDPs~\citep{puterman1994markov}. Learning in this setting inherits the difficulty of controlling bias spans in average-reward MDPs, since transition estimation errors depend on the variation of value functions across states. Algorithms for average-reward MDPs address this difficulty through span regularization or constraints~\citep{bartlett2009regal,fruit2018efficient}, or through value function clipping~\citep{hong2025reinforcement}.

Constraints on average costs, however, introduce further obstacles, making the analysis of weakly communicating CMDPs more challenging. An optimal constrained policy need not be greedy with respect to its reward action-value function, because a greedy policy may violate the constraint~\citep{ghosh2023achieving}. Consequently, the Bellman optimality arguments used to control spans in unconstrained optimistic planning do not carry over directly~\citep{chen2022learning}. A common approach is to incorporate the constraint into the reward through a dual multiplier. However, updating this multiplier changes the reward used for planning and hence the value functions. Controlling this additional variation complicates the analysis of existing primal-dual algorithms~\citep{ghosh2023achieving,yu2026learning}. For example, in the finite-horizon approximation of \citet{yu2026learning}, updating the multiplier once per episode of length $H$ yields a dual-regret bound of order $\sqrt{HT}$ after balancing the step-size terms. This term exceeds order $\sqrt{T}$ when $H$ grows with $T$; see Appendix~\ref{app:discussion}.

In fact, existing results for weakly communicating CMDPs either require computationally inefficient algorithms or have suboptimal bounds on regret and constraint violation. In the tabular setting, \citet{chen2022learning} solve a linear program under a finite-horizon approximation and achieves $\tbigO(T^{2/3})$ bounds on both quantities after $T$ interactions. They give another algorithm that improves the bounds to $\tbigO(\sqrt{T})$ by adding span constraints, but these constraints make the planning problem nonconvex. For CMDPs with linear function approximation, \citet{ghosh2023achieving} obtain $\tbigO(\sqrt{T})$ bounds with a computationally inefficient algorithm, and $\tbigO(T^{3/4})$ bounds with an efficient primal-dual algorithm. More recently, \citet{yu2026learning} establish strong duality for weakly communicating CMDPs and improve the efficient bounds in the linear setting to $\tbigO(T^{2/3})$.

Efficient $\tbigO(\sqrt{T})$ regret guarantees exist under stronger assumptions, including ergodicity~\citep{agarwal2022concave,chen2022learning,ghosh2023achieving}. Posterior-sampling algorithms also achieve $\tbigO(\sqrt{T})$ Bayesian regret and constraint violation for ergodic or communicating CMDPs~\citep{agarwal2022regret,provodin2024efficient}. Results for general policy parameterizations consider ergodic or unichain models and include approximation-error terms that can grow linearly with $T$~\citep{bai2024learning,satheesh2026regret}. These differences in assumptions and performance criteria are summarized in \Cref{tab:comparison} in Appendix~\ref{app:related work}.

These results leave open whether a computationally efficient algorithm can achieve $\tbigO(\sqrt{T})$ regret and cumulative constraint violation for weakly communicating average-reward CMDPs. In this work, we resolve this question in the tabular setting. Our contributions are summarized as follows.
\begin{itemize}
    \item We propose State-Augmented Clipped Value Iteration with Upper Confidence Bound (SA-CVI-UCB, \Cref{alg:main}) that achieves $\tbigO(\sqrt{T})$ regret and constraint violation with high probability (\Cref{thm:main}). To the best of our knowledge, this is the first computationally efficient algorithm with $\tbigO(\sqrt{T})$ guarantees for weakly communicating tabular average-reward CMDPs. The $\sqrt{T}$ dependence is optimal up to logarithmic factors when regret and constraint violation are controlled simultaneously~\citep{jaksch2010near,singh2023learning}. With a direct implementation of backward value iteration, the algorithm uses $\bigO(S^2AT^{5/2})$ arithmetic operations over $T$ interactions.

\item We develop a state-augmented formulation for regret analysis in average-reward CMDPs, combining a state variable that tracks cumulative constraint violation with a new reshaped reward based on a Huber potential. State augmentation has previously been used with safety budgets to enforce almost-sure constraints~\citep{sootla2022saute}. In our formulation, the added state determines the penalty on further violations while the reward function on the augmented state space remains fixed. Since the transition of the added state is known and deterministic, only the original transition kernel needs to be estimated. We define the reshaped reward through differences of the Huber potential, whose bounded slope keeps the per-step reward bounded. The potential differences telescope, so the cumulative reshaped reward equals the original cumulative reward minus the terminal potential. Together, state augmentation and the reshaped reward allow us to apply techniques used for unconstrained average-reward MDPs without worsening the regret rate in $T$. Specifically, we combine finite-horizon approximation~\citep{wei2021learning,chen2022learning,yu2026learning} with optimistic value iteration and value function clipping~\citep{hong2025reinforcement,chae2025learning} to obtain $\tbigO(\sqrt{T})$ regret and constraint violation in the original CMDP.
\end{itemize}

\section{Problem Setting}
\label{sec:problem-setting}

\paragraph{Average-Reward CMDPs}
We consider a constrained Markov decision process (CMDP) $\calM=(\calS,\calA,P,s_1,r,g)$, where $\calS$ and $\calA$ are finite state and action spaces, respectively, $P$ is the transition kernel, $s_1\in\calS$ is the initial state, and $r,g:\calS\times\calA\to\mathbb{R}$ are the reward and constraint functions, respectively. Here, $P(s'\mid s,a)$ denotes the probability of transitioning from state $s$ to $s'$ after taking action $a$. We assume that $r(s,a)\in[0,1]$ and $g(s,a)\in[-1,1]$ for all $(s,a)\in\calS\times\calA$. Moreover, we assume that $r$ and $g$ are known and deterministic, while the transition kernel $P$ is unknown.

Let $\Pi$ denote the set of stationary randomized policies. For any $\pi\in\Pi$ and initial state $s\in\calS$, we define the average reward, or gain, as $J_r^\pi(s)=\lim_{T\to\infty}(1/T)\bbE_\pi[\sum_{t=1}^T r(s_t,a_t)\mid s_1=s]$, where the expectation is taken over the trajectory generated by $P$ and $\pi$. We define $J_g^\pi(s)$ analogously by replacing $r$ with $g$. Given the initial state $s_1$, the average-reward CMDP is formulated as
\begin{align*}
    \sup_{\pi\in\Pi}\quad J_r^\pi(s_1)
    \qstq
    J_g^\pi(s_1)\geq 0.
\end{align*}
The agent interacts with the unknown CMDP for $T$ steps, starting from $s_1$. At each step $t\in[T]$, the agent observes the current state $s_t$, chooses an action $a_t\in\calA$ based on $a_t \sim \pi_t(\cdot|s_t)$, where $\pi_t$ denotes the agent's policy. Subsequently, the agent observes the next state $s_{t+1}\sim P(\cdot\mid s_t,a_t)$. We evaluate the learning algorithm in terms of regret and constraint violation, defined as
\begin{align*}
    \Regret(T)
    &=
    \sum_{t=1}^T \bigl(J_r^*-r(s_t,a_t)\bigr),
    \qquad
    \Violation(T)
    =
    \sum_{t=1}^T \bigl(-g(s_t,a_t)\bigr).
\end{align*}
Here, $\Regret(T)$ measures the cumulative reward gap relative to the optimal stationary policy, while $\Violation(T)$ measures the cumulative constraint violation. Finally, we impose the standard Slater condition, which ensures strict feasibility of the CMDP~\citep{singh2023learning,yu2026learning}.
\begin{assumption}[Slater condition]
\label{ass:slater}
There exists a stationary policy $\bar\pi\in\Pi$ such that $J_g^{\bar\pi}(s_1)\geq\gamma$ for some Slater constant $\gamma>0$. We assume that $\gamma$ is known, while the Slater policy $\bar\pi$ is unknown.
\end{assumption}

\paragraph{Weakly Communicating CMDPs}
Throughout the paper, we assume that the underlying MDP is weakly communicating. In particular, an MDP is weakly communicating if its state space can be partitioned into two sets such that all states in the first set communicate with one another under some stationary deterministic policies, while every state in the second set is transient under every stationary policy~\citep{puterman1994markov}. This assumption is weaker than the commonly imposed unichain and ergodic assumptions.

Unlike unconstrained weakly communicating MDPs, a weakly communicating CMDP does not necessarily admit an optimal stationary policy. Following previous works~\citep{chen2022learning,ghosh2023achieving,yu2026learning}, we therefore impose the following assumption.

\begin{assumption}[Existence of an optimal stationary policy]
\label{ass:optimal-policy}
There exists an optimal stationary policy $\pi^*\in\argmax_{\pi\in\Pi}\{J_r^\pi(s_1):J_g^\pi(s_1)\geq0\}$ such that its reward and constraint gains are independent of the initial state. Namely, there exist constants $J_r^*$ and $J_g^*$ such that $J_r^{\pi^*}(s)=J_r^*$ and $J_g^{\pi^*}(s)=J_g^*$ for all $s\in\calS$.
\end{assumption}

We denote by $v_r^*$ and $v_g^*$ the bias functions associated with $\pi^*$ and the reward and constraint functions, respectively, defined as $v_r^*(s)=\lim_{T\to\infty}(1/T)\sum_{t=1}^T\bbE_{\pi^*}[\sum_{i=1}^t(r(s_i,a_i)-J_r^*)\mid s_1=s]$ and $v_g^*(s)=\lim_{T\to\infty}(1/T)\sum_{t=1}^T\bbE_{\pi^*}[\sum_{i=1}^t(g(s_i,a_i)-J_g^*)\mid s_1=s]$. For any $V:\calS\to\mathbb{R}$, we define its span as $\spn(V)=\max_{s\in\calS}V(s)-\min_{s\in\calS}V(s)$.

\section{Algorithm}\label{sec:algorithm}
In this section, we propose a computationally efficient algorithm for learning weakly communicating average-reward CMDPs that achieves rate-optimal guarantees in $T$. Our algorithm can be viewed as casting the original CMDP into an augmented MDP and then applying a learning algorithm to it. Importantly, the main technical novelty lies in the design of the augmented MDP: (i) state augmentation for regret analysis in the average-reward setting, and (ii) a new reshaped reward based on a Huber potential.

Before presenting our approach, we discuss the limitations of occupancy-measure optimization and primal-dual methods for weakly communicating average-reward CMDPs. Existing high-probability guarantees in this setting either require computationally inefficient planning or have suboptimal dependence on $T$. For instance, Algorithm~4 in \citet{chen2022learning} achieves rate-optimal guarantees by imposing explicit constraints on the spans of the value functions induced by the occupancy measures. These constraints make the feasible region generally nonconvex, and no efficient method for solving the resulting optimization problem is known. This motivates an approach that is both computationally efficient and rate-optimal in $T$.

Another widely used approach is to incorporate the constraint into the reward through a dual multiplier $z_k$, leading to the composite reward $r + z_k g$---called the primal-dual method. However, its direct extension to the average-reward setting becomes nontrivial. In particular, rate-optimal algorithms for unconstrained average-reward MDPs rely critically on controlling the deviation of value functions across iterations~\citep{hong2025a}. Once the dual multiplier is introduced, the variation of $z_k$ induces additional variation in the composite reward and hence in the value functions, which makes the same rate-optimal analysis difficult in the constrained setting; see \Cref{app:discussion} for more details. This suggests that incorporating the constraint information into the reward leads to suboptimal guarantees due to the additional variation in the composite reward.

\subsection{State-Augmented MDP}
Motivated by the failure of primal-dual approaches, our method can be summarized as follows: to avoid the additional variation induced by incorporating the constraint information into the reward, we instead incorporate it into the state space, inspired by \citet{sootla2022saute}. Intuitively, this avoids the additional variation while preserving the desired statistical guarantees, since the evolution of the safety state is predictable from the current augmented state and action; see \eqref{eq:tilde P}.

Technically, the augmented (unconstrained) MDP $\widetilde\calM = (\widetilde \calS, \calA, \widetilde P, (s_1,z_1), \tilde r)$ is defined as follows. We introduce a safety state $z_t$ and augment the original state space to $\widetilde{\calS}=\calS\times\calZ$, where $\calZ = \{\Delta n: n\in \bbZ, |\Delta n|\leq 2T\}$ with $\Delta = 1/\sqrt{T}$. Given $z_1=0$, the safety state evolves as $z_{t+1}=\Pi_{\calZ}(z_t-g(s_t,a_t))$, where $\Pi_{\calZ}$ returns the closest element in $\calZ$. For simplicity, let $\psi(s,z,a)=\Pi_{\calZ}(z-g(s,a))$. We then define the augmented transition kernel and reshaped reward as
\begin{align}
\widetilde P(s',z' \mid s,z,a) &= P(s'\mid s,a)\cdot\Ind\{z'=\psi(s,z,a)\},\label{eq:tilde P}\\
\tilde r(s,z,a) &= r(s,a)+\Phi_W(z)-\Phi_W(\psi(s,z,a)), \label{eq:tilde r}
\end{align}
for each $(s,z,a)\in\calS\times\calZ\times\calA$ and $(s',z')\in\calS\times\calZ$. This defines an unconstrained MDP over the augmented state space $\widetilde{\calS}$. Here, $\Phi_W:\bbR\to\bbR_+$ is a nondecreasing potential function, whose formal definition will be provided later.

The safety state $z_t$ and the reshaped reward $\tilde r$ can be interpreted as follows. Intuitively, $z_t$ tracks the cumulative constraint violation up to step $t$: ignoring the projection, $z_t\approx-\sum_{\tau=1}^{t-1}g(s_\tau,a_\tau)$, so a large positive $z_t$ indicates accumulated constraint violation, whereas a negative $z_t$ indicates accumulated constraint surplus. Accordingly, since $\Phi_W$ is nondecreasing, the potential difference term in \eqref{eq:tilde r} acts as a regularization term for constraint satisfaction: actions that decrease the safety state receive a nonnegative regularization term (i.e., a bonus), whereas actions that increase it receive a nonpositive one (i.e., a penalty). Thus, this term plays a role analogous to the penalty induced by a dual multiplier, with the magnitude of the regularization determined by the current safety state. Consequently, optimizing the reshaped reward naturally balances reward maximization and constraint satisfaction, motivating an unconstrained learning problem over the augmented MDP.

\paragraph{Motivation of $\tilde r$}
We describe the design motivation for $\tilde r$, as it plays a critical role in our analysis. From a technical perspective, $\tilde r$ provides a natural bridge from the regret analysis of the augmented MDP to that of the original CMDP. In particular, its potential-difference form in \eqref{eq:tilde r} admits the telescoping structure, i.e., $\sum_{t=1}^T \bigl(\Phi_W(z_t)-\Phi_W(z_{t+1})\bigr) = -\Phi_W(z_{T+1})$, since $z_1=0$ and $\Phi_W(0)=0$. Based on this, the cumulative reshaped reward can be expressed as
\begin{align}\label{eq:telescope}
\sum_{t=1}^T \tilde r(s_t,z_t,a_t) = \sum_{t=1}^T r(s_t,a_t) - \Phi_W(z_{T+1}).
\end{align}
This relation allows us to translate a regret bound for the augmented MDP, expressed in terms of $\tilde r(s_t,z_t,a_t)$, into a bound on the original regret together with the terminal potential, which serves as a useful ingredient for controlling both regret and violation; see Step 2 of \Cref{sec:analysis}.

\paragraph{Choice of $\Phi_W$: Huber Potential}
Previously, we discussed that the potential-difference structure in the design of $\tilde r$ provides a bridge between the augmented MDP and the original CMDP in the analysis. The remaining question is whether we can successfully carry out the regret analysis for the augmented MDP. To this end, several regularity conditions are typically required; for example, the reshaped reward $\tilde r$ should be sufficiently bounded. However, for common choices of potential functions such as quadratic or exponential potentials, $\tilde r$ may become large because these potentials are not globally Lipschitz. In particular, since the safety state $z$ can grow with $T$, the potential difference $\Phi(z)-\Phi(\psi(s,z,a))$ may also become large, leading to an unfavorable dependence on $T$ in the regret analysis.

This observation motivates us to choose a potential function with favorable regularity properties, such as Lipschitzness. For this purpose, we use the Huber potential $\Phi_W$, defined as follows: for constants $\Lambda, W > 0$,
\begin{equation}
\Phi_W(x)=
\begin{cases}
0, & x\le 0,\\
\Lambda x^2/(2W), & 0<x\le W,\\
\Lambda\left(x-W/2\right), & x>W.
\end{cases} \label{eq:PhiW}
\end{equation}

For comparison of potential functions, since $\Phi_W(x)$ grows linearly when $x>W$, it is globally Lipschitz, whereas quadratic and exponential potentials are not. Besides, the Huber potential satisfies additional properties that are useful in our analysis, which are summarized in the following lemma.
\begin{lemma}\label{lem:PhiW}
    Let $\Phi_W:\bbR \to \bbR$ be defined as in  \eqref{eq:PhiW}. Then we have the following properties:
    \begin{enumerate}
        \item $0 \leq \Phi_W'(x) \leq \Lambda$ for all $x\in\bbR$.
        \item $\Phi_W'(y)x \leq \Phi_W(x+y) - \Phi_W(y) \leq \Phi_W'(y)x + \Lambda x^2/ (2W)$ for all $x,y \in \bbR$.
        \item For any $\lambda \in [0,\Lambda)$ and $x \in \bbR$, we have $(\Lambda- \lambda)[x]_+\leq \Phi_W(x) - \lambda x + \Lambda W/2$.
    \end{enumerate}
\end{lemma}

\paragraph{Comparison of Augmented RL}
While our state augmentation is related to that of \citet{sootla2022saute}, the main difference in the resulting augmented MDP lies in the reshaped reward. In both approaches, the augmented state tracks cumulative constraint information, up to a difference in the sign convention for the safety state $z_t$. For the reshaped reward, \citet{sootla2022saute} imposes an infinite regularization once the constraint is violated. In a reward-maximization form, their reshaped objective can be expressed as
\begin{align*}
    \tilde r(s,z,a)
    =
    \begin{cases}
        r(s,a), & z \geq 0, \\
        -\infty, & z < 0,
    \end{cases}
\end{align*}
thereby targeting almost-sure constraint satisfaction. In contrast, our reshaped reward~\eqref{eq:tilde r} uses the potential difference $\Phi_W(z)-\Phi_W(\psi(s,z,a))$, which remains uniformly bounded due to the Lipschitzness of the Huber potential. This choice is tailored to our regret analysis, allowing us to control the augmented value functions while retaining the telescoping structure needed to bound regret and constraint violation.

\subsection{Description of Algorithm}
\begin{algorithm}[t]
\caption{{S}tate-{A}ugmented {C}lipped {V}alue {I}teration with {UCB} (SA-CVI-UCB)}
\label{alg:main}
\textbf{Input:} the optimism parameter $\beta$; the number of steps $T$; the horizon length $H$ and $K = T/H$; the Slater constant $\gamma$; the clipping parameter $C$ and span $\spn(v_r^*), \spn(v_g^*)$; the parameter for Lyapunov function $W$; the discretization gap $\Delta$; \\
\textbf{Initialize:} $z_1^1 \leftarrow 0$; $V_{k,H+1}(s,z) \leftarrow 0$; $\calZ=\{\Delta n: n\in\bbZ, |\Delta n|\leq 2T\}$; \\
$\hat P_1(s'|s,a) \leftarrow 1/S$ and $N_1(s,a,s'), N_1(s,a) \leftarrow 0 \quad \forall (s,a,s')\in\calS\times\calA\times\calS$; \\
$\tilde r(s,a,z) \leftarrow r(s,a) + \Phi_W(z) - \Phi_W(\psi(s,z,a)) \quad \forall (s,a,z) \in \calS\times\calA \times\calZ$;
\begin{algorithmic}[1]
    \For{$k=1, \ldots, K$}
        \For{$h=H,\ldots,1$}
            \For{$(s,a,z) \in \calS \times\calA \times\calZ$}
                \State $Q_{k,h}(s,a,z) \leftarrow \tilde r(s,a,z) + \hat P_k(\cdot|s,a)^\top V_{k,h+1}(\cdot,\psi(s,a,z)) + \beta/\sqrt{N_k(s,a)\vee 1}$\label{line:Q}
                \State $\widetilde V_{k,h}(s,z) \leftarrow  \max_{a\in\calA} Q_{k,h}(s,a,z)$ \label{line:tilde V}
                \State $V_{k,h}(s,z) \leftarrow \widetilde V_{k,h}(s,z) \wedge (\min_{s'\in\calS} \widetilde V_{k,h}(s',z) + 2C)$ \label{line:clip}
            \EndFor
        \EndFor
        \State $N_{k+1}(\cdot,\cdot,\cdot)\leftarrow N_k(\cdot,\cdot,\cdot)$ and $N_{k+1}(\cdot,\cdot)\leftarrow N_k(\cdot,\cdot)$
        \For{$h=1,\ldots, H$} \label{line:execute 0}
            \State Take $a_h^k \in \argmax_{a\in\calA} Q_{k,h}(s_h^k,a,z_h^k)$
            \State Observe $s_{h+1}^k \sim P(\cdot|s_h^k,a_h^k)$
            \State $N_{k+1}(s_h^k, a_h^k,s_{h+1}^k) \leftarrow N_{k+1}(s_h^k, a_h^k,s_{h+1}^k) + 1$ 
            \State $N_{k+1}(s_h^k,a_h^k) \leftarrow N_{k+1}(s_h^k,a_h^k) + 1$
            \State $z_{h+1}^k \leftarrow \psi(s_h^k,z_h^k,a_h^k)$
        \EndFor
        \State $s_1^{k+1} \leftarrow s_{H+1}^k$ and $z_1^{k+1} \leftarrow z_{H+1}^k$
        \State
        $\hat P_{k+1}(s'|s,a) \leftarrow \begin{cases}
            N_{k+1}(s,a,s')/N_{k+1}(s,a) & \text{if } N_k(s,a) > 0, \\
            1/S & \text{if } N_k(s,a) = 0, 
        \end{cases} \quad \forall (s,a,s') \in \calS\times\calA\times\calS$ \label{line:execute 1}
    \EndFor
\end{algorithmic}
\end{algorithm}
We present the proposed \underline{S}tate-\underline{A}ugmented \underline{C}lipped \underline{V}alue \underline{I}teration with \underline{U}pper \underline{C}onfidence \underline{B}ound (SA-CVI-UCB) in \Cref{alg:main}. The algorithm adopts the finite-horizon approximation framework with $K$ episodes and horizon $H$, so that $T=KH$. Accordingly, we use the trajectory notation $s_h^k, a_h^k$ instead of $s_t, a_t$, with the correspondence $t=(k-1)H+h$. The algorithmic parameter choices, including $K$ and $H$, are summarized in \eqref{eq:alg param}. In each episode $k\in[K]$, the algorithm first performs backward value iteration over the augmented state space $\widetilde{\calS}=\calS\times\calZ$ using the empirical transition kernel $\hat P_k$. It then executes the greedy policy for $H$ steps while updating the safety state according to $\psi$, and finally updates the empirical transition model.

In Lines~\ref{line:Q}--\ref{line:clip}, the backward recursion constructs an optimistic value estimate for the augmented MDP. Importantly, the true transition kernel $P(\cdot\mid s,a)$ is unknown, while the next safety state $\psi(s,a,z)$ is deterministic. Hence, $\hat P_k$ is estimated using all visits to $(s,a)$, independently of the safety state, and the bonus $\beta/\sqrt{N_k(s,a)\vee1}$ accounts only for uncertainty in $P$.

In Line~\ref{line:clip}, for each fixed $z\in\calZ$, the algorithm clips the value-function estimate across the original state space. As shown in \Cref{sec:analysis}, the comparator value function in the augmented MDP has span at most $2C$ for every fixed safety state. Therefore, clipping at this level preserves optimism while allowing the statistical error to depend on the tighter span parameter $C$, whereas without clipping one would rely on the na\"ive span bound of $H$ for $\widetilde V_{k,h}$. This step is crucial for obtaining the desired regret bound.

After planning, in Lines~\ref{line:execute 0}--\ref{line:execute 1}, the algorithm executes the action $a_h^k\in\arg\max_a Q_{k,h}(s_h^k,a,z_h^k)$ and observes the next state $s_{h+1}^k$. The safety state is then updated deterministically as $z_{h+1}^k=\psi(s_h^k,a_h^k,z_h^k)$. Both components of the augmented state are carried across episode boundaries, and the empirical transition kernel is updated using the cumulative transition counts.

\subsection{Main Result}
We present our main result. We choose the algorithmic parameters as follows:
\begin{align}\label{eq:alg param}
\begin{aligned}
    &\delta \in (0,1/3), \ 
    \Delta = 1/\sqrt{T}, \ 
    \Lambda = 2/\gamma,\  
    W=K=H=\sqrt{T}, \\
    &C = \spn(v_r^*) + \Lambda \spn(v_g^*) + \frac{\Lambda}{W}\left(\spn(v_g^*)^2 + 2H(\spn(v_g^*)+2)^2\right)+ \frac{H\Lambda\Delta}{2}, \\
    &\beta = 2CS\sqrt{\log(2TS^2|\calA|/\delta)/2}.
\end{aligned}
\end{align}
Moreover, let $\spn(v_{\lambda^*}^*)$ denote the span of the optimal bias function with respect to the composite reward $r+\lambda^*g$. where $\lambda^*$ is the optimal dual multiplier; see Appendix~\ref{app:strong duality}. Note that this quantity appears only in the analysis and it not required by the algorithm. Then, we have the following result.
\begin{theorem}\label{thm:main}
    For any $\delta\in(0,1/3)$, with probability at least $1-3\delta$, we have
    \begin{align*}
        \Regret(T)&\leq \tbigO\left(\left(\spn(v_r^*)+\frac{1}{\gamma}(\spn(v_g^*)+\spn(v_g^*)^2+1)\right)S^2A\sqrt{T}\right), \\
        \Violation(T)&\leq \tbigO\left((\spn(v_r^*) +\spn(v_g^*)+\spn(v_g^*)^2)S^2A\sqrt{T}+\spn(v_{\lambda^*}^*)\sqrt{T}\right).
    \end{align*}
\end{theorem}
The theorem shows that SA-CVI-UCB achieves $\tbigO(\sqrt{T})$ regret and constraint violation, up to problem-dependent factors involving $S$, $A$, $\gamma$, and the relevant bias spans. In particular, these results establish SA-CVI-UCB as the first computationally efficient algorithm to achieve rate-optimal guarantees in this setting. Here, the running time of SA-CVI-UCB is polynomial in $S$, $A$, and $T$. Under the parameter choice~\eqref{eq:alg param}, we have $|\calZ|\leq 4T^{3/2}+1$, and in each episode, Lines~\ref{line:Q}--\ref{line:clip} compute $Q_{k,h}(s,a,z)$ for all $(h,s,a,z)\in[H]\times\calS\times\calA\times\calZ$ with $\bigO(S)$ operations each. Hence, the total running time over $K$ episodes is $\bigO(KHS^2A|\calZ|)=\bigO(S^2AT^{5/2})$.

\paragraph{Why does the augmentation technique work?}
Some readers may wonder, in principle, why the state-augmentation technique leads to rate-optimal guarantees in the average-reward setting. To address this, we highlight two useful properties of the augmented formulation in our setting. First, $\tilde r$ is a fixed reward function on the augmented state space, while the constraint information evolves through the Markovian safety state. Hence, the additional reward variation and value function deviation induced by a changing dual multiplier is avoided. Second, since $g$ is known and deterministic, the transition of the safety state is also known. Therefore, in the tabular setting, state augmentation introduces no additional unknown transition dynamics, and the learner only needs to estimate the original transition kernel $P(\cdot\mid s,a)$. These properties make the augmented formulation particularly suitable for obtaining rate-optimal guarantees in the average-reward setting.

\section{Analysis}\label{sec:analysis}
In this section, we present the analysis of \Cref{thm:main}. To establish bounds on $\Regret(T)$ and $\Violation(T)$, our proof proceeds in the following three-step framework. The detailed proofs are deferred to Appendix~\ref{app:proofs}.
\begin{align*}
\underbrace{\RegAug(T)}_{\textbf{Step 1}}
\ \xLongrightarrow[\textbf{Step 2}]{}
\Regret(T) + \Phi_W(z_{T+1})
\ \xLongrightarrow[\textbf{Step 3}]{} 
\Regret(T), \Violation(T).
\end{align*}
First, we establish a regret bound for the augmented MDP under the reshaped reward $\tilde r$ and transition kernel $\widetilde P$. We denote this regret by $\RegAug(T)$ and formally define it as $\RegAug(T) = \sum_{k=1}^K (V_{\tilde r,1}^{\pi^*}(s_1^k, z_1^k) - \sum_{h=1}^H \tilde r(s_h^k,a_h^k,z_h^k))$, where $V_{\tilde r,1}^{\pi^*}(s,z) = \bbE_{\widetilde P,\pi^*}[\sum_{h=1}^H \tilde r(s_h,a_h,z_h)\mid s_1=s, z_1=z]$. Second, we transfer this bound to the original CMDP to obtain a bound on $\Regret(T)+\Phi_W(z_{T+1})$. Finally, we use this bound to separately control $\Regret(T)$ and $\Violation(T)$.

Before going through each step, we introduce the key lemma underlying our analysis. Roughly, this lemma ensures that the two quantities---$\spn(V_{\tilde r, h}^{\pi^*}(\cdot,z))$ and $HJ_r^* - V_{\tilde r, 1}^{\pi^*}(s,z)$---are small, as $C=\tbigO(1)$ under the parameter choice~\eqref{eq:alg param}. For comparison, this lemma can be viewed as an analogue of Lemma 2 in \citet{wei2020model}, which considers the simpler setting of average-reward MDPs without state augmentation.
\begin{lemma}\label{lem:span}
    Let $C$ be defined in \eqref{eq:alg param}. Then we have, for all $(h,s,z)\in [H]\times\calS \times \calZ$,
    \begin{enumerate}
        \item $\spn(V_{\tilde r, h}^{\pi^*}(\cdot,z)) \leq 2C$,
        \item $HJ_r^* - V_{\tilde r, 1}^{\pi^*}(s,z) \leq C$.
    \end{enumerate}
\end{lemma}
Each property plays a crucial role in establishing Steps 1 and 2, respectively. In particular, the first statement ensures that the benchmark value function appearing in $\RegAug(T)$ has a uniformly bounded span, which is essential for controlling the regret in the augmented MDP. Moreover, the second statement provides a bridge between $\RegAug(T)$ and $\Regret(T)$, since $V_{\tilde r,1}^{\pi^*}$ and $J_r^*$ serve as the respective benchmarks for these two regret metrics. Given this lemma, we are now ready to proceed through each step.

\paragraph{Step 1: Analysis of the Augmented MDP}
To bound $\RegAug(T)$, we first introduce the following decomposition:
\begin{align*}
\RegAug(T)
&\leq \sum_{k=1}^K \left(V_{\tilde r,1}^{\pi^*}(s_1^k,z_1^k) - V_{k,1}(s_1^k,z_1^k)\right) \\
&+ \sum_{k=1}^K\sum_{h=1}^H\left(P(\cdot|s_h^k,a_h^k)^\top V_{k,h+1}(\cdot,z_{h+1}^k) - V_{k,h+1}(s_{h+1}^k,z_{h+1}^k)\right) \\
&+ \sum_{k=1}^K\sum_{h=1}^H \frac{2\beta}{\sqrt{N_k(s_h^k,a_h^k)\vee 1}}.
\end{align*}
Here, the second and third summations on the right-hand side can be bounded by $\tbigO(C\sqrt{T})$ and $\tbigO(\beta\sqrt{SAT})$, respectively, using standard arguments. The first summation is nonpositive by the following optimism lemma.
\begin{lemma}[Optimism]\label{lem:optimism}
Let $\delta\in(0,1)$. With probability at least $1-\delta$, simultaneously for all $(k,h,s,z) \in [K]\times [H]\times \calS \times\calZ$, we have $V_{\tilde r, h}^{\pi^*}(s,z) \leq V_{k,h}(s,z)$.
\end{lemma}
The first statement of \Cref{lem:span} plays a crucial role in establishing this lemma. In particular, a standard UCB analysis ensures that $V_{\tilde r,h}^{\pi^*}(s,z) \leq \widetilde V_{k,h}(s,z)$, where $\widetilde V_{k,h}(s,z)$ denotes the unclipped value function estimate. The first statement of \Cref{lem:span} then ensures that this optimism property is preserved even after clipping. As a consequence, combining these bounds yields
\begin{equation}\label{eq:RegAug}
    \RegAug(T) \leq \tbigO(\sqrt{T}).
\end{equation}

\paragraph{Step 2: Transfer to the Original CMDP}
Next, based on the bound on $\RegAug(T)$, we establish a bound on $\Regret(T) + \Phi_W(z_{T+1})$. To this end, by the telescoping structure induced by the design of $\tilde r$, we have
\begin{align*}
\RegAug(T)
&= \Regret(T) + \Phi_W(z_{T+1}) + \sum_{k=1}^K \left(V_{\tilde r,1}^{\pi^*}(s_1^k,z_1^k) - HJ_r^*\right).
\end{align*}
Here, the summation $\sum_{k=1}^K \left(V_{\tilde r,1}^{\pi^*}(s_1^k,z_1^k) - HJ_r^*\right)$ on the right-hand side represents the conversion error between the augmented MDP and the original CMDP. By the second statement of \Cref{lem:span}, this error term can be lower bounded by $-CK$. Since $K=\sqrt{T}$, combining this with the bound on $\RegAug(T) \leq \tbigO(\sqrt{T})$ yields
\begin{align}\label{eq:Regret + PhiW main}
\Regret(T) + \Phi_W(z_{T+1}) \leq \tbigO(\sqrt{T}).
\end{align}

\paragraph{Step 3: Obtaining $\Regret(T)$ and $\Violation(T)$}
To conclude the analysis, we derive separate bounds on regret and constraint violation using \eqref{eq:Regret + PhiW main}. In particular, since $\Phi_W$ is nonnegative, the bound on $\Regret(T)$ follows directly. We therefore focus on obtaining a bound on $\Violation(T)$.

This step, however, is nontrivial for the Huber potential. For comparison, in the case of a quadratic or exponential potential, it suffices to na\"ively lower bound $\Regret(T)\geq -T$. Combining this with \eqref{eq:Regret + PhiW main} then yields $z_{T+1}\leq \tbigO(\sqrt{T})$, which concludes the analysis since $z_{T+1}$ approximates $\Violation(T)$ up to the discretization error $\Delta T$. In contrast, since the Huber potential $\Phi_W(x)$ grows linearly when $x\geq W$, the same argument no longer yields the desired bound.

To overcome this issue, our strategy is as follows. Instead of using the na\"ive lower bound $\Regret(T)\geq -T$, we exploit a sharper lower bound derived from strong duality; see Appendix~\ref{app:strong duality} for details. In particular,
\begin{align*}
-\lambda^* z_{T+1} \lesssim \Regret(T),
\end{align*}
where $\lambda^*$ denotes the optimal dual variable associated with the original CMDP, which is guaranteed to satisfy $\lambda^* \leq 1/\gamma$, and $\lesssim$ denotes an informal inequality used to emphasize intuition. Combining this lower bound with \eqref{eq:Regret + PhiW main} yields
\begin{align*}
\Phi_W(z_{T+1}) - \lambda^* z_{T+1} \leq \tbigO(\sqrt{T}).
\end{align*}
To derive a bound on $z_{T+1}$ from this relation, we use the third statement of \Cref{lem:PhiW}, which states that $(\Lambda-\lambda^*)[z_{T+1}]_+ \leq \Phi_W(z_{T+1}) - \lambda^* z_{T+1} + \Lambda W/2$. Combining this with the facts that $\Lambda-\lambda^* \geq 1/\gamma$ and $W = \sqrt{T}$ yields
\begin{align}\label{eq:zT main}
z_{T+1} \leq \tbigO(\sqrt{T}).
\end{align}
Finally, again, since $z_{T+1}$ approximates $\Violation(T)$ up to the discretization error, we obtain $\Violation(T) \leq \tbigO(\sqrt{T})$.

\section{Numerical Experiments}\label{sec:experiments}

We evaluate \Cref{alg:main} on a tabular CMDP obtained by modifying the CMDP of \citet{yu2026learning}. In the early states of the chain, the reward is given to the action that lowers the constraint value, so the policy that maximizes the reward violates the constraint and the algorithm must trade off the reward against the constraint. We compare \Cref{alg:main} with four baselines, (i) PD-LSCVI-UCB \citep{yu2026learning}, (ii) Algorithm 3 of \citet{chen2022learning}, (iii) Algorithm 2 of \citet{ghosh2023achieving}, and (iv) a random policy. We repeat $10$ simulations with $T = 32{,}000$ steps and report the cumulative regret and constraint violation, which are summarized in \Cref{fig:experiments}. \Cref{alg:main} achieves sublinear regret and constraint violation, whereas the constraint violation of Algorithm 2 of \citet{ghosh2023achieving} grows linearly. Compared with PD-LSCVI-UCB, \Cref{alg:main} achieves lower regret and constraint violation. Algorithm 3 of \citet{chen2022learning} attains negative regret at the cost of a constraint violation that grows linearly.

We also examine the sensitivity of each algorithm to the scale of its bonus term, and find that \Cref{alg:main} achieves sublinear regret and constraint violation for every scale $c \leq 1$, whereas every baseline incurs linear regret or constraint violation for some $c$. The results are shown in \Cref{fig:sweep}. Additional experiments and experimental details are deferred to \Cref{app:experiments}.

\begin{figure}[t]
    \centering
    \includegraphics[width=0.89\linewidth]{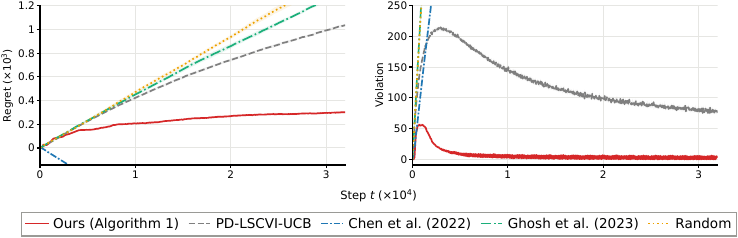}
    \vspace{-6pt}
    \caption{Plots of regret and constraint violation for $T = 32{,}000$ steps. Each plot represents the average over $10$ trials, and shaded regions indicate $95\%$ confidence intervals.}
    \label{fig:experiments}
\end{figure}

\section{Conclusion}
\label{sec:conclusion}

We studied infinite-horizon average-reward CMDPs under the weakly communicating assumption. Our algorithm, SA-CVI-UCB, is computationally efficient and achieves rate-optimal $\tbigO(T^{1/2})$ regret and constraint violation in the tabular setting. The main idea is to incorporate cumulative constraint violation into the state space and design the reshaped reward using a Huber potential. This formulation keeps the reshaped reward bounded while avoiding the additional reward variation caused by changing dual multipliers. The algorithm combines this formulation with finite-horizon approximation, optimistic value iteration, and value function clipping. It remains open whether the same computational efficiency and rate-optimal guarantees can be achieved for weakly communicating average-reward linear CMDPs. We refer the reader to \Cref{app:discussion} for a detailed discussion of this extension and the limitations of the primal-dual approach.

\bibliography{ref}
\bibliographystyle{plainnat}

\newpage
\appendix
\newpage

\section{Related Work}\label{app:related work}
\begin{table}[htbp]
\centering
\caption{Comparison of guarantees for learning infinite-horizon average-reward CMDPs with unknown transitions after $T$ steps. Problem-dependent factors are suppressed. Bounds hold with high probability unless marked as expected or Bayesian guarantees. WC denotes weak communication; the WC results assume an optimal stationary policy with state-independent gains. Additional assumptions and differences in performance criteria are specified below.}
\label{tab:comparison}
\small
\setlength{\tabcolsep}{4pt}
\begin{tabular}{llllll}
\toprule
Algorithm & Setting & Assumption & Regret & Violation & Efficient \\
\midrule
\citet{wei2022provably}$^{a}$ & Tabular & Ergodic & $\tbigO(T^{5/6})$ & $0$ & \yes \\
\citet{agarwal2022concave}$^{a,b}$ & Tabular & Ergodic & $\tbigO(\sqrt{T})$ & $0$ & \yes \\
\citet{agarwal2022regret}$^{c}$ & Tabular & Ergodic & $\tbigO(\sqrt{T})$ & $\tbigO(\sqrt{T})$ & \yes \\
\citet{chen2022learning} (Alg.~1) & Tabular & Ergodic & $\tbigO(\sqrt{T})$ & $\tbigO(1)$ & \yes \\
\citet{singh2023learning} & Tabular & Unichain$^{d}$ & $\tbigO(T^{2/3})$ & $\tbigO(T^{2/3})$ & \yes$^{e}$ \\
\citet{ghosh2023achieving} (Alg.~3)$^{f}$ & Linear & Ergodic & $\tbigO(\sqrt{T})$ & $\tbigO(\sqrt{T})$ & \yes \\
\citet{provodin2024efficient}$^{c,g}$ & Tabular & Communicating & $\tbigO(\sqrt{T})$ & $\tbigO(\sqrt{T})$ & \yes \\
\midrule
\citet{bai2024learning}$^{h}$ & General & Ergodic & $\tbigO(T^{4/5})$ & $\tbigO(T^{4/5})$ & \yes \\
\citet{satheesh2026regret}$^{i}$ & General & Unichain & $\tbigO(\sqrt{T})$ & $\tbigO(\sqrt{T})$ & \yes \\
\midrule
\citet{chen2022learning} (Alg.~3) & Tabular & WC & $\tbigO(T^{2/3})$ & $\tbigO(T^{2/3})$ & \yes \\
\citet{chen2022learning} (Alg.~4) & Tabular & WC & $\tbigO(\sqrt{T})$ & $\tbigO(\sqrt{T})$ & \no \\
\citet{ghosh2023achieving} (Alg.~1)$^{j}$ & Linear & WC & $\tbigO(\sqrt{T})$ & $\tbigO(\sqrt{T})$ & \no \\
\citet{ghosh2023achieving} (Alg.~2) & Linear & WC & $\tbigO(T^{3/4})$ & $\tbigO(T^{3/4})$ & \yes \\
\citet{yu2026learning} & Linear & WC & $\tbigO(T^{2/3})$ & $\tbigO(T^{2/3})$ & \yes \\
\midrule
\textbf{Ours} (\Cref{alg:main}) & Tabular & WC & $\tbigO(\sqrt{T})$ & $\tbigO(\sqrt{T})$ & \yes \\
\bottomrule
\end{tabular}

\vspace{3pt}
\begin{minipage}{\linewidth}
\footnotesize
$^{a}$Zero aggregate violation holds under the stated sufficiently-large-$T$ conditions. For \citet{wei2022provably}, both bounds hold in expectation, and zero violation means nonpositive expected signed cumulative violation, not safety at every step. $^{b}$The row reports UC-CURL, which covers Lipschitz concave objectives and convex constraints, including standard CMDPs, under a Slater condition. $^{c}$Bayesian guarantees, averaging over the transition prior and interaction randomness, under the stated strict-feasibility and horizon conditions. $^{d}$Bounds hold in expectation. In addition to unichain structure, the analysis uses uniform geometric convergence and a finite worst-policy state-to-state hitting time. $^{e}$The optimistic planning problem admits an LP reformulation using state-action-next-state occupancy measures. $^{f}$Requires uniform mixing and a uniformly positive definite stationary feature Gram matrix under every stationary policy (Assumptions~4 and~5). $^{g}$Requires a prior supported on models with bounded diameter, a strictly feasible policy inducing an irreducible and aperiodic chain, and sufficiently large $T$. $^{h}$Expected trajectory guarantees under general policy parameterization, with smoothness and Fisher non-degeneracy assumptions. The displayed rates omit the policy approximation term. $^{i}$Expected guarantees under general policy parameterization, with bounded Lipschitz scores, Fisher non-degeneracy, critic approximation and uniform hitting-time assumptions. The displayed rates omit the policy and critic approximation terms. The violation bound uses average costs of policy iterates rather than costs realized along the trajectory. $^{j}$Requires the optimal policy to belong to a class of soft-max policies.
\end{minipage}
\end{table}

\paragraph{Average-Reward CMDPs}
Average-reward CMDPs model continuing tasks in which an agent maximizes long-term average reward subject to long-term average constraints~\citep{altman2021constrained}. Model-based methods establish finite-time regret guarantees for average-reward CMDPs~\citep{singh2023learning,chen2022learning}. Model-free algorithms also achieve sublinear expected regret and nonpositive expected cumulative constraint violation for sufficiently large $T$~\citep{wei2022provably}. Under the ergodic assumption, \citet{chen2022learning} obtain $\tbigO(\sqrt{T})$ regret and $\tbigO(1)$ constraint violation. Under the weakly communicating assumption, their efficient finite-horizon approximation gives $\tbigO(T^{2/3})$ bounds for both quantities. They also achieve $\tbigO(\sqrt{T})$ bounds by imposing span constraints on the finite-horizon value functions induced by the occupancy measure, but the resulting feasible region is generally nonconvex and no efficient solution method is known. For linear CMDPs, the efficient finite-horizon algorithm of \citet{ghosh2023achieving} obtains $\tbigO(T^{3/4})$ bounds. \citet{yu2026learning} establish strong duality for weakly communicating CMDPs and improve these bounds to $\tbigO(T^{2/3})$ through primal-dual clipped value iteration. Posterior-sampling methods also provide Bayesian regret guarantees for communicating CMDPs~\citep{provodin2024efficient}. Under the ergodic assumption, general policy parameterizations have been studied through primal-dual policy gradient~\citep{bai2024learning} and actor-critic methods~\citep{xu2025global}. \citet{satheesh2026regret} analyze actor-critic methods under the unichain assumption. Our work obtains computational efficiency together with $\tbigO(\sqrt{T})$ regret and constraint violation with high probability for weakly communicating tabular CMDPs.

Our algorithm also builds on techniques for unconstrained average-reward MDPs. In the tabular setting, optimistic planning gives regret guarantees for communicating and weakly communicating MDPs~\citep{jaksch2010near,bartlett2009regal,zhang2019regret,boone2024achieving}. \citet{fruit2018efficient} obtain efficient exploration under the weakly communicating assumption by controlling the bias span. Model-free approaches combine optimistic Q-learning with discounted-reward approximation~\citep{wei2020model,zhang2023sharper}. For linear MDPs, \citet{wei2021learning} develop efficient value-iteration algorithms based on finite-horizon approximation. \citet{hong2025reinforcement} combine discounted-reward approximation with value function clipping to achieve $\tbigO(\sqrt{T})$ regret. \citet{hong2025a} further remove the dependence of the computational complexity on the state-space size through efficient clipping and deviation-controlled value iteration. We adapt optimistic value iteration and value function clipping to a state-augmented MDP. The augmented formulation incorporates cumulative constraint violation into the state and avoids the additional reward variation caused by changing dual multipliers, which complicates the primal-dual extension of these techniques to CMDPs. More broadly, unconstrained average-reward MDPs have been studied extensively, with computationally efficient algorithms achieving minimax-optimal regret up to logarithmic factors in the tabular setting. We refer the reader to \citet{boone2024achieving} and the references therein for further discussion of this literature.

\paragraph{CMDPs via State Augmentation}
State augmentation allows policies to use constraint information that is not contained in the original state. One approach incorporates Lagrange multipliers into the state space. \citet{calvo2024state} develop this approach for average-reward CMDPs, learning policies conditioned on the multipliers and updating the multipliers during execution to obtain asymptotic feasibility and near-optimality in expectation. This formulation retains dual updates and has also been extended to distributed multi-agent problems~\citep{agorio2024multiagent,amayacorredor2026scalable}. Another approach augments the state with a budget or cumulative cost. For expected discounted-cost constraints, \citet{carrara2019budgeted} augment the state and action with a budget and establish a budgeted Bellman optimality equation. For episodic RL with knapsack constraints and finitely supported costs on a discrete grid, \citet{chen2021factored} track the remaining budgets and exploit the resulting factored transition structure to obtain $\tbigO(\sqrt{T})$ regret with near-optimal dependence on $S$, $A$, $H$, and $T$. The computational cost depends exponentially on the number of constraints.

Budget augmentation has also been studied under several notions of safety. Saut\'e RL~\citep{sootla2022saute} incorporates a remaining safety budget into the state and reshapes the objective to target almost-sure constraint satisfaction. Under the stated regularity assumptions, its optimal value functions converge to those of the infinite-penalty formulation as the violation penalty tends to infinity. Simmer~\citep{sootla2022enhancing} studies how to schedule the safety budget during training under expected-cost and almost-sure constraints, with empirical improvements in safe exploration. For terminating MDPs with binary damage, \citet{castellano2022almost} characterize the minimum safety budget. They give sample-complexity guarantees for learning this budget from a generative model, assuming a positive lower bound on nonzero probabilities in the joint transition and damage kernel. For anytime constraints, which must hold at every step almost surely, \citet{mcmahan2024anytime} use cumulative-cost augmentation to derive planning and learning algorithms, together with hardness and approximation results. Related constructions support equilibrium computation in anytime-constrained Markov games~\citep{mcmahan2025equilibria}. For probabilistic state-avoidance constraints, \citet{hameldelecourt2025probabilistic} construct a shield on an augmented MDP that guarantees safety and preserves the constrained optimum when the safety dynamics are known. Using value-demand augmentation, \citet{mcmahan2024deterministic} compute feasible policies that are near-optimal among deterministic policies under a time-space-recursive constraint. \citet{mcmahan2025approximability} use artificial budget variables to obtain bicriteria approximation guarantees that allow a small constraint violation. These works study different feasibility criteria or planning objectives from the regret and cumulative violation considered here.

\paragraph{Risk Criteria via State Augmentation}
State augmentation also helps handle nonlinear criteria applied to cumulative rewards or costs. \citet{xu2011probabilistic} track accumulated reward in probabilistic-goal and chance-constrained MDPs. For conditional value-at-risk constraints, \citet{chow2018risk} develop actor-critic methods on an augmented MDP and prove almost-sure convergence to locally optimal policies under their sampling and approximation assumptions. \citet{yang2024risk} instead condition policies on the quantile level of accumulated cost. For episodic CMDPs with an entropic-risk constraint, \citet{ghosh2025risk} introduce a continuous budget variable and use a primal-dual algorithm to obtain $\tbigO(K^{3/4})$ regret and constraint violation over $K$ episodes with high probability. \citet{wang2025a} reduce finite-horizon learning with optimized certainty equivalent objectives to risk-neutral learning in an augmented MDP.

\paragraph{Reward Penalties on Augmented States}
Closely related to our reward design, \citet{wang2023ccpo} combine safety-state augmentation with state-dependent reward penalties and an adaptive penalty multiplier in the discounted setting. \citet{jiang2024reward} combine cumulative-cost augmentation with trajectory-dependent reward penalties in finite-horizon problems. They give penalty conditions under which optimal policies of the augmented MDP satisfy expected-cost or chance constraints, and relate a modified penalty to excess-loss constraints. The discounted reward adjustments in \citet{jiang2024reward} also admit a telescoping representation. We study continuing, weakly communicating average-reward CMDPs and track cumulative constraint violation throughout learning. Our reshaped reward uses a Huber potential whose bounded derivative keeps the per-step reward adjustment bounded. We use this property and a quadratic remainder bound to control the augmented finite-horizon value functions under an optimal policy of the original CMDP. In particular, we bound their span across the original states uniformly over the added state and control the shortfall from the original average-reward benchmark (\Cref{lem:span}). These bounds allow us to apply finite-horizon approximation and optimistic value iteration with clipping. Together with a strong-duality argument, this yields a computationally efficient algorithm with $\tbigO(\sqrt{T})$ regret and constraint violation with high probability.

\section{Strong Duality}\label{app:strong duality}
In this section, we introduce several lemmas that are associated with the strong duality of weakly communicating average-reward CMDPs~\citep{yu2026learning}, which will be used for Step 3 of \Cref{sec:analysis}. We have the following results.
\begin{lemma}[Theorem 1 in \citet{yu2026learning}]\label{thm:strong duality}
    Let $\calM = (\calS, \calA, P, r,g)$ be a weakly communicating average-reward CMDP with $S, |\calA| < \infty$, and let $s_1 \in \calS$ be any initial state. Suppose that the CMDP $\calM$ is feasible. Then the following properties hold:
    \begin{enumerate}
        \item There exists $\lambda^*\geq 0$ such that $D(\lambda^*) = \inf_{\lambda\geq 0}D(\lambda)$.
        \item Strong duality, i.e., $J_r^*(s_1) = D(\lambda^*)$.
    \end{enumerate}
\end{lemma}
By utilizing the strong duality result, we can show the following lemma, which states optimal dual variable is bounded in terms of the Slater constant $\gamma$.
\begin{lemma}[Lemma 3 in \citet{yu2026learning}]\label{lem:optimal dual variable bound}
    Suppose that Slater's condition holds, i.e., there exists a stationary policy $\bar\pi$ such that $J_g^{\bar\pi}(s_1) \geq \gamma$ for some $\gamma > 0$. Let $\lambda^*$ denote the optimal dual variable defined in \Cref{thm:strong duality}. Then $\lambda^* \leq 1/\gamma$.
\end{lemma}
Given a bounded optimal dual variable $\lambda^*$, which is guaranteed by Lemmas~\ref{thm:strong duality} and~\ref{lem:optimal dual variable bound}, we consider an unconstrained average-reward MDP with respect to the reward function $r+\lambda^* g$. Then, the Bellman optimality condition implies the following: there exists $J_{\lambda^*}^* \in \bbR$ and $v_{\lambda^*}^*: \calS \to \bbR$ such that for all $(s,a) \in \calS\times\calA$,
\begin{align}\label{eq:bellman optimality equation}
    J_{\lambda^*}^* + v_{\lambda^*}^*(s) \geq r(s,a) + \lambda^*g(s,a) + Pv_{\lambda^*}^*(s,a)
\end{align}
Finally, we introduce the following lemma that states when Slater's condition holds, then $\Regret(T)$ can be lower bounded in terms of $\Violation(T)$.
\begin{lemma}[Lemma 7 in \citet{yu2026learning}] \label{lem:regret lower bound}
Let $v_{\lambda^*}^*$ be defined in \eqref{eq:bellman optimality equation}. With probability at least $1-\delta$,
\begin{align*}
    \Regret(T) \geq -\lambda^* \Violation(T) - \spn({v_{\lambda^*}^*})\sqrt{2T\log(1/\delta)} - \spn({v_{\lambda^*}^*}).
\end{align*}    
\end{lemma}

\section{Deferred Proofs of Sections~\ref{sec:algorithm} and ~\ref{sec:analysis}}\label{app:proofs}

\subsection{Proof of Lemma~\ref{lem:PhiW}}
\begin{lemma}[Restatement of Lemma~\ref{lem:PhiW}]\label{lem:PhiW app}
    Let $\Phi_W:\bbR \to \bbR$ be defined as in  \eqref{eq:PhiW}. Then we have the following properties:
    \begin{enumerate}
        \item $0 \leq \Phi_W'(x) \leq \Lambda$ for all $x\in\bbR$.
        \item $\Phi_W(x+y) - \Phi_W(y) \leq \Phi_W'(y)x + \frac{\Lambda}{2W}x^2$ for all $x,y \in \bbR$.
        \item $\Phi_W'(y)x \leq \Phi_W(x+y) - \Phi_W(y)$ for all $x,y \in \bbR$.
        \item For any $\lambda \in [0,\Lambda)$ and $x \in \bbR$, we have $\Phi_W(x) - \lambda x \geq (\Lambda- \lambda)[x]_+ - \frac{\Lambda W}{2}$.
    \end{enumerate}
\end{lemma}
\begin{proof}
    The first statement is clear, since we have
    \begin{align}\label{eq:deriv PhiW}
        \Phi_W'(x) = \begin{cases}
            0 & \text{if } x \leq 0, \\
            (\Lambda / W)x & \text{if }0< x \leq W, \\
            \Lambda & \text{if } W<x.
        \end{cases}
    \end{align}
    We prove the second and third statements. By \eqref{eq:deriv PhiW}, we have for any $x,y \in \bbR$ such that $x \leq y$,
    \begin{align}\label{eq:PhiW 1}
        0 \leq \Phi_W'(y) - \Phi_W'(x) \leq \frac{\Lambda}{W}(y-x)
    \end{align}
    While $0 \leq \Phi_W'(y) - \Phi_W'(x)$ is trivial, as $\Phi_W'$ is nondecreasing, the $\Phi_W'(y) - \Phi_W'(x) \leq (\Lambda /W)(y-x)$ can be shown by the following arguments. When $x,y \leq  0$ or $W \leq x,y$, the statement is trivial. When $x \leq 0 < y \leq W$, we have $\Phi_W'(y) - \Phi_W'(x) = (\Lambda/W)y \leq (\Lambda/W)(y-x)$. When $x < 0 \leq W < y$, we have $\Phi_W'(y) - \Phi_W'(x) = \Lambda \leq (\Lambda/W)(y-x)$. Finally, when $0 < x \leq W < y$, we have $\Phi_W'(y) - \Phi_W'(x) \leq \Lambda - (\Lambda/W)x \leq (\Lambda/W)(y-x)$. Therefore, \eqref{eq:PhiW 1} is true. Then it follows that
    \begin{align*}
        0 \leq \int_0^1 \left(\Phi_W'(y + \tau x) - \Phi_W'(y) \right) x d\tau \leq \int_0^1 \frac{\Lambda}{W}\tau x^2 d\tau = \frac{\Lambda}{2W}x^2
    \end{align*}
    where the inequality follows from \eqref{eq:deriv PhiW}. Moreover, $\int_0^1 \left(\Phi_W'(y + \tau x) - \Phi_W'(y) \right) x d\tau$ can be rewritten as follows.
    \begin{align*}
        \int_0^1 \left(\Phi_W'(y + \tau x) - \Phi_W'(y) \right) x d\tau 
        &= \int_0^1 \Phi_W'(y+\tau x) xd\tau - \Phi_W'(y) x \\
        &= \int_0^x \Phi_W'(y+\tau') d\tau' - \Phi_W'(y) x \\
        &= \Phi_W(y+x)- \Phi_W(y) - \Phi_W'(y)x
    \end{align*}
    where the second equality follows from the substitution $\tau' = x\tau$, and the last equality follows from the fundamental theorem of calculus. Combining these leads to
    \begin{align*}
        0 \leq \Phi_W(y+x) - \Phi_W(y) -\Phi_W'(y) x \leq \frac{\Lambda}{2W}x^2.
    \end{align*}
    Note that this implies the second and third statements.

    Consider $\lambda \in [0, \Lambda)$. To prove the fourth statement, we consider the following three cases: (i) $x \leq 0$, (ii) $0 < x \leq W$, and (iii) $W <x$. For (i), we have $\Phi_W(x) - \lambda x = -\lambda x \geq 0 \geq -\frac{\Lambda W}{2} = (\Lambda-\lambda)[x]_+ -\frac{\Lambda W}{2}$. For (ii), we have
    \begin{align*}
        \Phi_W(x) - \lambda x - (\Lambda-\lambda)[x]_+ + \frac{\Lambda W}{2} 
        &= \frac{\Lambda}{2W}x^2 -\Lambda x + \frac{\Lambda W}{2}= \frac{\Lambda}{2W}(x-W)^2 \geq 0.
    \end{align*}
    Therefore, $\Phi_W(x) - \lambda x \geq (\Lambda-\lambda)[x]_+ - \frac{\Lambda W}{2}$ when (ii). For (iii), we have
    \begin{align*}
        \Phi_W(x) - \lambda x = \Lambda(x-\frac{W}{2}) -\lambda x = (\Lambda -\lambda)[x]_+ -\frac{\Lambda W}{2}.
    \end{align*}
    Finally, these imply that the third statement holds.
\end{proof}

\subsection{Proof of Lemma~\ref{lem:span}}
We first prove the following lemma, which is useful to show \Cref{lem:span}.
\begin{lemma} \label{lem:J-V}
    Define $C$ as
    \begin{align}\label{eq:def C}
        C = \spn(v_r^*) + \Lambda \spn(v_g^*) + \frac{\Lambda}{W}\left(\spn(v_g^*)^2 + 2H(\spn(v_g^*)+2)^2\right)+ \frac{H\Lambda\Delta}{2}.
    \end{align}
    Then we have, for all $(h,s,z)\in [H]\times\calS \times \calZ$,
    \begin{align*}
        \left|(H-h+1)J_r^* + \Phi_W(z) - \Phi_W\left(z-(H-h+1)J_g^*\right) - V_{\tilde r,h}^{\pi^*}(s,z)\right| \leq C.
    \end{align*}
\end{lemma}
\begin{proof}
    Recall that $\tilde P(s',z'|s,a,z) = \Ind\{z'=\Pi_\calZ(z-g(s,a))\}\cdot P(s'|s,a)$. Hence, given $s_{h'},a_{h'}, z_{h'}$, we have $z_{h'+1} = \Pi_\calZ(z_{h'} - g(s_{h'},a_{h'}))$ for each $h'$. Then it follows that
    \begin{align}\label{eq:J-V 0}
    \begin{aligned}
        &V_{\tilde r,h}^{\pi^*}(s,z) \\
        &\quad= \bbE_{\tilde P,\pi^*}\left[\sum_{h'=h}^H \tilde r(s_{h'},a_{h'},z_{h'}) | s_h = s, z_h=z\right]\\
        &\quad= \bbE_{\tilde P,\pi^*}\left[\sum_{h'=h}^H \left(r(s_{h'}, a_{h'}) +\Phi_W(z_{h'})-\Phi_W\left(\Pi_\calZ(z_{h'}-g(s_{h'},a_{h'}))\right) \right)| s_h = s, z_h=z\right]\\
        &\quad= \bbE_{\tilde P,\pi^*}\left[\sum_{h'=h}^H \left(r(s_{h'}, a_{h'}) +\Phi_W(z_{h'})-\Phi_W(z_{h'+1}) \right)| s_h = s, z_h=z\right]\\
        &\quad= \bbE_{\tilde P,\pi^*}\left[\sum_{h'=h}^H r(s_{h'}, a_{h'}) + \Phi_W(z_h)-\Phi_W(z_{H+1})| s_h = s, z_h=z\right].
    \end{aligned}
    \end{align}
    Moreover, since $z_{h'+1} = \Pi_\calZ(z_{h'} - g(s_{h'},a_{h'}))$, we have for all $h'\in[H]$,
    \begin{align*}
        \left|z_{h'+1} - (z_{h'} - g(s_{h'},a_{h'})) \right|  \leq \frac{\Delta}{2}.
    \end{align*}
    Then it follows that
    \begin{align*}
        \left|z_{H+1} - \left(z_h -\sum_{h'=h}^H g(s_{h'},a_{h'})\right)\right| 
        &= \left|\sum_{h'=h}^H \left(z_{h'+1} - (z_{h'} - g(s_{h'},a_{h'}))\right)\right|\\
        &\leq \sum_{h'=h}^H \left|z_{h'+1} - (z_{h'} - g(s_{h'},a_{h'}))\right|\\
        &\leq \frac{(H-h+1)\Delta}{2}.
    \end{align*}
    It leads to
    \begin{align*}
        \left|\Phi_W(z_{H+1}) - \Phi_W\left(z_h - \sum_{h'=h}^H g(s_{h'},a_{h'})\right)\right| 
        &\leq \Lambda  \left|z_{H+1} - \left(z_h -\sum_{h'=h}^H g(s_{h'},a_{h'})\right)\right| \\
        &\leq \frac{(H-h+1)\Lambda\Delta}{2}
    \end{align*}
    where the first inequality follows from the Lipschitzness of $\Phi_W$, i.e., $|\Phi_W(y) - \Phi_W(x)| \leq \Lambda |y-x|$. Then we have
    \begin{align}\label{eq:J-V 1}
    \begin{aligned}
        \bbE_{\tilde P, \pi^*} \left[ \Phi_W(z_{H+1}) | s_h=s,z_h=z\right] 
        &\leq \bbE_{\tilde P, \pi^*}\left[\Phi_W\left(z_h-\sum_{h'=h}^H g(s_{h'},a_{h'})\right) | s_h=s,z_h=z\right]\\
        &\quad+ \frac{(H-h+1)\Lambda\Delta}{2}.
    \end{aligned}
    \end{align}
    By applying \eqref{eq:J-V 1} to \eqref{eq:J-V 0}, we have
    \begin{align}\label{eq:J-V V}
    \begin{aligned}
        V_{\tilde r,h}^{\pi^*}(s,z)  
        &\leq \bbE_{\tilde P,\pi^*}\left[\sum_{h'=h}^H r(s_{h'}, a_{h'})|s_h=s, z_h=z\right] \\
        &\quad+ \bbE_{\tilde P,\pi^*}\left[\Phi_W(z_h)-\Phi_W\left(z_h-\sum_{h'=h}^H g(s_{h'},a_{h'})\right)| s_h = s, z_h=z\right]\\
        &\quad+ \frac{(H-h+1)\Lambda\Delta}{2}\\
        &= \bbE_{P,\pi^*}\left[\sum_{h'=h}^H r(s_{h'}, a_{h'})|s_h=s\right] \\
        &\quad + \Phi_W(z) - \bbE_{P,\pi^*}\left[\Phi_W\left(z-\sum_{h'=h}^H g(s_{h'},a_{h'})\right)| s_h = s\right]\\
        &\quad+ \frac{(H-h+1)\Lambda\Delta}{2}
    \end{aligned}
    \end{align}
    where the second equality is true since terms in the expectations do not include $\{z_{h'}\}_{h'=h+1}^{H+1}$, it can be rewritten using $P$ rather than $\tilde P$ along with the fact that $\bbE_{\tilde P,\pi^*}[\Phi_W(z_h)|z_{h}=z] = \Phi_W(z)$. By applying \eqref{eq:J-V V} to the desired term, we have
    \begin{align}\label{eq:J-V -1}
    \begin{aligned}
        &\left|(H-h+1)J_r^* + \Phi_W(z) - \Phi_W(z-(H-h+1)J_g^*) - V_{\tilde r,h}^{\pi^*}(s,z)\right|\\
        &\leq \underbrace{\left|(H-h+1)J_r^* - \bbE_{P,\pi^*}\left[\sum_{h'=h}^H r(s_{h'}, a_{h'})|s_h=s\right]\right|}_{\text{(I)}} \\
        &\quad+ \underbrace{\left|\bbE_{P,\pi^*}\left[\Phi_W\left(z-\sum_{h'=h}^H g(s_{h'},a_{h'})\right)| s_h = s\right] - \Phi_W(z-(H-h+1)J_g^*)\right|}_{\text{(II)}}\\
        &\quad+ \frac{(H-h+1)\Lambda\Delta}{2}.
    \end{aligned}
    \end{align}
    For (I), by \Cref{lem:chen2022 lemma 4}, we have
    \begin{align*}
        \text{(I)} \leq \spn(v_r^*).
    \end{align*}
    Next, we bound (II). To show this, recall that for all $s\in\calS$,
    \begin{align*}
        J_g^* + v_g^*(s) = \sum_{a\in\calA}\pi^*(a|s)(g(s,a) + Pv_g^*(s,a)).
    \end{align*}
    Then it follows that for each $h' \in [H]$,
    \begin{align}\label{eq:J-V xi}
        \bbE_{P,\pi^*}\left[\underbrace{g(s_{h'},a_{h'}) + v_g^*(s_{h'+1}) - v_g^*(s_{h'}) - J_g^*}_{\triangleq \xi_{h'}}| \calF_{h'} \right] = 0
    \end{align}
    where $\calF_{h'} = \sigma(s_1,a_1, \ldots, a_{h'-1}, s_{h'})$. Moreover, we define $\xi_{h'} = g(s_{h'},a_{h'}) + v_g^*(s_{h'+1}) - v_g^*(s_{h'}) - J_g^*$. By summing $\xi_{h'}$ over $h'=h,\ldots,H$, we have
    \begin{align*}
        \sum_{h'=h}^H \xi_{h'} = \sum_{h'=h}^H g(s_{h'},a_{h'}) + v_g^*(s_{H+1}) - v_g^*(s_{h}) - (H-h+1)J_g^*.
    \end{align*}
    It can be rewritten as for any $z\in\calZ$ and $h \in [H]$,
    \begin{align*}
        z - \sum_{h'=h}^H g(s_{h'},a_{h'}) + v_g^*(s_h) - v_g^*(s_{H+1}) + \sum_{h'=h}^H \xi_{h'} = z - (H-h+1)J_g^*.
    \end{align*}
    Now, we apply the second statement of \Cref{lem:PhiW app}. Then it follows that
    \begin{align*}
        &\Phi_W\left(z-\sum_{h'=h}^H g(s_{h'},a_{h'})\right) - \Phi_W(z - (H-h+1)J_g^*)\\
        &\leq \Phi_W'\left(z-(H-h+1)J_g^*\right) \left(v_g^*(s_{H+1}) - v_g^*(s_h) - \sum_{h'=h}^H \xi_{h'}\right) \\
        &\quad+ \frac{\Lambda}{2W}\left(v_g^*(s_{H+1}) - v_g^*(s_h) - \sum_{h'=h}^H \xi_{h'}\right)^2
    \end{align*}
    By taking $\bbE_{P,\pi^*}[\cdot| s_h=s]$, we have
    \begin{align}\label{eq:J-V 2}
    \begin{aligned}
        &\bbE_{P,\pi^*}\left[\Phi_W\left(z-\sum_{h'=h}^H g(s_{h'},a_{h'})\right)|s_h=s\right] - \Phi_W(z - (H-h+1)J_g^*) \\
        &\leq \Phi_W'\left(z-(H-h+1)J_g^*\right) \underbrace{\bbE_{P,\pi^*}\left[v_g^*(s_{H+1}) - v_g^*(s) - \sum_{h'=h}^H \xi_{h'}| s_h=s\right]}_{\text{(III)}} \\
        &\quad+ \frac{\Lambda}{2W} \underbrace{\bbE_{P,\pi^*}\left[\left(v_g^*(s_{H+1}) - v_g^*(s_h) - \sum_{h'=h}^H \xi_{h'}\right)^2|s_h=s\right]}_{\text{(IV)}}.
    \end{aligned}
    \end{align}
    Next, we bound (III) and (IV) individually. For (III), note that
    \begin{align*}
        \bbE_{P,\pi^*}\left[\sum_{h'=h}^H \xi_{h'} |s_h=s\right] = \bbE_{P,\pi^*}\left[\sum_{h'=h}^H \bbE_{P,\pi^*}[\xi_{h'}|\calF_{h'}] |s_h=s\right] = 0
    \end{align*}
    where the first equality follows from the tower rule, and the second equality follows from \eqref{eq:J-V xi}. Moreover, we have $v_g^*(s_{H+1}) - v_g^*(s) \leq \spn(v_g^*)$ almost surely. Then we have
    \begin{align*}
        |\text{(III)}| \leq \spn(v_g^*).
    \end{align*}
    To bound (IV), we have
    \begin{align*}
        \text{(IV)} 
        &\leq 2\bbE_{P,\pi^*}\left[\spn(v_g^*)^2 + \left(\sum_{h'=h}^H \xi_{h'}\right)^2|s_h=s\right]\\
        &= 2\spn(v_g^*)^2 + 2\bbE_{P,\pi^*}\left[\sum_{h'=h}^H \xi_{h'}^2|s_h=s\right]\\
        &\leq 2\spn(v_g^*)^2 + 2(H-h+1) (\spn(v_g^*) + 2)^2
    \end{align*}
    where the first inequality follows from the fact that $(x+y)^2 \leq 2x^2 + 2y^2$ for all $x,y\in\bbR$, and the last inequality is due to $|\xi_{h'}| \leq 2 + \spn(v_g)^*$ for all $h'\in[H]$, as $J_g^* \in [-1,1]$ and $g(s_{h'},a_{h'}) \in [-1,1]$. Note that the equality can be shown by the following argument.
    \begin{align*}
        \bbE_{P,\pi^*}\left[\left(\sum_{h'=h}^H \xi_{h'}\right)^2|s_h=s\right] 
        &= \bbE_{P,\pi^*}\left[\sum_{h'=h}^H \xi_{h'}^2 + \sum_{\substack{h\leq h'',h'''\leq H \\ h'' \neq h'''}} \xi_{h''}\xi_{h'''}|s_h=s\right]\\
        &= \bbE_{P,\pi^*}\left[\sum_{h'=h}^H \xi_{h'}^2|s_h=s\right]
    \end{align*}
    where the second equality follows from for any $h''$ and $h'''$ such that $h'' < h'''$, we have $\bbE_{P,\pi^*}[\xi_{h''}\xi_{h'''}] = \bbE_{P,\pi^*}[\xi_{h''}\bbE_{P,\pi^*}[\xi_{h'''}|\calF_{h'''}]] = 0$ by the tower rule. By applying bounds on (III) and (IV) to \eqref{eq:J-V 2}, we have
    \begin{align}\label{eq:J-V 3}
    \begin{aligned}
        &\bbE_{P,\pi^*}\left[\Phi_W\left(z-\sum_{h'=h}^H g(s_{h'},a_{h'})\right)|s_h=s\right] - \Phi_W(z - (H-h+1)J_g^*) \\
        &\leq \Phi_W'\left(z-(H-h+1)J_g^*\right)|\text{(III)}| + \text{(IV)}\\
        &\leq \Lambda \spn(v_g^*) + \frac{\Lambda}{W}\left(\spn(v_g^*)^2 + 2(H-h+1)(\spn(v_g^*)+2)^2\right).
    \end{aligned}
    \end{align}
    By applying the same argument with the third statement of \Cref{lem:PhiW app}, we have
    \begin{align}\label{eq:J-V 4}
    \begin{aligned}
        \bbE_{P,\pi^*}\left[\Phi_W\left(z-\sum_{h'=h}^H g(s_{h'},a_{h'})\right)|s_h=s\right] - \Phi_W(z - (H-h+1)J_g^*) \geq -\Lambda \spn(v_g^*)
    \end{aligned}
    \end{align}
    where the only differences are the direction of inequality, and the absence of (IV). Then by \eqref{eq:J-V 3} and \eqref{eq:J-V 4}, we have
    \begin{align*}
        \text{(II)} 
        &= \left|\bbE_{P,\pi^*}\left[\Phi_W\left(z-\sum_{h'=h}^H g(s_{h'},a_{h'})\right)\right] - \Phi_W(z - (H-h+1)J_g^*)\right|\\
        &\leq \Lambda \spn(v_g^*) + \frac{\Lambda}{W}\left(\spn(v_g^*)^2 + 2(H-h+1)(\spn(v_g^*)+2)^2\right).
    \end{align*}
    Finally, by applying bounds on (I) and (II) to \eqref{eq:J-V -1}, we have
    \begin{align*}
        &\left|(H-h+1)J_r^* + \Phi_W(z) - \Phi_W(z-(H-h+1)J_g^*) - V_{\tilde r,h}^{\pi^*}(s,z)\right| \\
        &\leq \spn(v_r^*) + \Lambda \spn(v_g^*) + \frac{\Lambda}{W}\left(\spn(v_g^*)^2 + 2(H-h+1)(\spn(v_g^*)+2)^2\right) + \frac{(H-h+1)\Lambda\Delta}{2}\\
        &\leq C.
    \end{align*}
    This concludes the proof of the first statement.
\end{proof}

\begin{lemma}[Restatement of Lemma~\ref{lem:span}]\label{lem:span app}
    Let $C$ defined in \eqref{eq:def C}. Then we have, for all $(h,s,z)\in [H]\times\calS \times \calZ$,
    \begin{enumerate}
        \item $\spn(V_{\tilde r, h}^{\pi^*}(\cdot,z)) \leq 2C$,
        \item $HJ_r^* - V_{\tilde r, 1}^{\pi^*}(s,z) \leq C$.
    \end{enumerate}
\end{lemma}

\begin{proof}
    We prove the first statement. For a fixed $z\in\calZ$ and any $s', s'' \in \calS$, we have
    \begin{align*}
        &|V_{\tilde r,h}^{\pi^*}(s',z) - V_{\tilde r,h}^{\pi^*}(s'',z)|\\
        &\leq \left|(H-h+1)J_r^* + \Phi_W(z) - \Phi_W(z-(H-h+1)J_g^*) - V_{\tilde r,h}^{\pi^*}(s',z)\right| \\
        &\quad+\left|(H-h+1)J_r^* + \Phi_W(z) - \Phi_W(z-(H-h+1)J_g^*) - V_{\tilde r,h}^{\pi^*}(s'',z)\right| \\
        &\leq 2C
    \end{align*}
    where the first inequality follows from the triangle inequality, and the second inequality follows from \eqref{lem:J-V}. Then it follows that
    \begin{align*}
        \spn(V_{\tilde r,h}^{\pi^*}(\cdot,z)) 
        &= \max_{s',s''} \left|V_{\tilde r,h}^{\pi^*}(s',z) - V_{\tilde r,h}^{\pi^*}(s'',z)\right| \leq 2C.
    \end{align*}
    Next, we prove the second statement. By \Cref{lem:J-V},
    \begin{align*}
        (H-h+1)J_r^* + \Phi_W(z) - \Phi_W(z-(H-h+1)J_g^*) - V_{\tilde r,h}^{\pi^*}(s,z) \leq C.
    \end{align*}
    Note that since $J_g^* \geq 0$, we have $z \geq z- (H-h+1)J_g^*$. Then since $\Phi_W$ is a nondecreasing function, we have $\Phi_W(z) - \Phi_W(z-(H-h+1)J_g^*) \geq 0.$ Therefore, it follows that
    \begin{align*}
        (H-h+1)J_r^* - V_{\tilde r,h}^{\pi^*}(s,z) \leq C.
    \end{align*}
    \qed
\end{proof}

\subsection{Deferred proofs for Step 1 of Section~\ref{sec:analysis}}
\begin{lemma} \label{lem:z in Z}
Suppose that $\Delta \leq 1$. We have $z_h^k \in \calZ$ for all $(k,h) \in [K] \times [H]$.
\end{lemma}
\begin{proof}
     For all $(k,h) \in [K] \times [H]$, recall that $z_{h+1}^k = \Pi_\calZ(z_h^k - g(s_h^k, a_h^k))$. Moreover, note that
     \begin{align*}
         |z_{h+1}^k - z_h^k| 
         &= \left|\Pi_\calZ(z_h^k - g(s_h^k,a_h^k)) - z_h^k\right|\\
         &= \left|\Pi_\calZ(z_h^k - g(s_h^k,a_h^k)) - (z_h^k- g(s_h^k,a_h^k)) + g(s_h^k,a_h^k)\right|\\
         &\leq \left|\Pi_\calZ(z_h^k - g(s_h^k,a_h^k)) - (z_h^k- g(s_h^k,a_h^k))\right| + \left|g(s_h^k,a_h^k)\right|\\
         &\leq \frac{\Delta}{2} + 1
     \end{align*}
     where the last inequality follows from the error due to the projection onto the grid with gap $\Delta$ is at most $\Delta/2$, and the fact that $g(s_h^k,a_h^k) \in [-1,1].$ Recall that $z_{H+1}^k = z_h^{k+1}$ and $z_1^1=0$. Then we have for all $(k,h) \in [K]\times [H]$,
     \begin{align*}
         |z_{h+1}^k| 
         &= \left|\sum_{k=1}^K\sum_{h=1}^H (z_{h+1}^k - z_h^k)\right|\\
         &\leq \sum_{k=1}^K\sum_{h=1}^H\left|z_{h+1}^k - z_h^k\right|\\
         &\leq \frac{\Delta T}{2} + T \\
         &\leq 2T.
     \end{align*}
    Then it implies that $|z_h^k| \leq 2T$ for all $(k,h) \in [K]\times [H]$. Since $z_h^k$ is obtained from $\Pi_\calZ$, we have
     \begin{align*}
         z_h^k \in\{\Delta n: n \in \bbZ, |\Delta n| \leq 2T\} = \calZ.
     \end{align*}
\end{proof}

\begin{lemma}\label{lem:concentration}
    Suppose that $S\geq 2$. With probability at least $1-\delta$, for all $(k,h,s,a,z) \in [K]\times [H] \times \calS \times \calA \times \calZ$, we have
    \begin{align*}
        \left|\left(\hat P_k(\cdot|s,a) - P(\cdot|s,a)\right)^\top V_{k,h+1}(\cdot, \psi(s,a,z))\right| \leq \frac{\beta}{\sqrt{N_k(s,a)\vee 1}}
    \end{align*}
    where $\beta = 2CS\sqrt{\log(2TS^2|\calA|/\delta)/2}$.
\end{lemma}
\begin{proof}
    Fix $(k,h,s,a,z) \in [K]\times [H] \times \calS \times \calA \times \calZ$. By the algorithm, we have $\sum_{s'}\hat P_k(s'|s,a) =1$. Then we have
    \begin{align}\label{eq:concen 1}
        \left|\left(\hat P_k(\cdot|s,a) - P(\cdot|s,a)\right)^\top V_{k,h+1}(\cdot, \psi(s,a,z))\right| \leq 2C \left\|\hat P_k(\cdot|s,a) - P(\cdot|s,a)\right\|_1
    \end{align}
    where the inequality follows from \Cref{lem:(p-q)v}, and the fact that $\spn(V_{k,h}(\cdot, \psi(s,a,z))) \leq 2C$.

    Next, we consider the following two cases: (i) $N_k(s,a) = 0$ and (ii) $N_k(s,a) > 0$. For (i), the desired statement becomes trivial, i.e., when $S\geq 2$,
    \begin{align*}
        2C\left\|\hat P_k(\cdot|s,a) - P(\cdot|s,a)\right\|_1 \leq 4C \leq \frac{\beta}{\sqrt{N_k(s,a) \vee 1}}.
    \end{align*} 
    For (ii), it can be shown as follows. Fix $k,s,a,s'$. Let $n = N_k(s,a)$, and define the random process $\{X_i\}_{i=1}^n$ such that
    \begin{align*}
        X_i = \begin{cases}
            1 &\text{if $s'$ is visited immediately after the $i$th visit to $(s,a)$},\\
            0 &\text{otherwise}.
        \end{cases}
    \end{align*}
    Then, $\{X_i\}_{i=1}^n$ are i.i.d. random variables that satisfy $\bbE[X_i] = P(s'|s,a)$ for all $i$. By Hoeffding's inequality, with probability at least $1-\delta$,
    \begin{align*}
        \left|\frac{1}{n}\sum_{i=1}^n X_i - P(s'|s,a)\right| \leq \sqrt{\frac{\log (2/\delta)}{2n}}.
    \end{align*}
    By letting $\delta \leftarrow \delta / (TS^2|\calA|)$ and applying a union bound over $n,s,a,s'$, we have
    \begin{align*}
        \left|\frac{1}{n}\sum_{i=1}^n X_i - P(s'|s,a)\right| \leq \sqrt{\frac{\log (2TS^2|\calA|/\delta)}{2n}}, \ \ \forall (n,s,a,s') \in [T]\times\calS\times\calA\times\calS.
    \end{align*}
    Note that the above inequality holds for all $n \in [T]$, we can substitute $n = N_k(s,a)$. Moreover, we have $(1/n)\sum_{i=1}^n X_i = N_k(s,a,s') / (N_k(s,a)\vee 1) = \hat P_k(s'|s,a)$. Thus, it follows that
    \begin{align*}
        |\hat P_k(s'|s,a) - P(s'|s,a)| \leq \sqrt{\frac{\log (2TS^2|\calA|/\delta)}{2N_k(s,a)}}, \ \ \forall (k,s,a,s') \in [K]\times\calS\times\calA\times\calS.
    \end{align*}
    By summing this over $s'\in\calS$, we have
    \begin{align*}
        \left\|\hat P_k(\cdot|s,a) - P(\cdot|s,a)\right\|_1 \leq S\sqrt{\frac{\log(2KS^2|\calA|/\delta)}{2N_k(s,a)}}, \ \ \forall (k,s,a) \in [K]\times\calS\times\calA.
    \end{align*}
    Finally, for both cases (i) and (ii), we have shown that, with probability at least $1-\delta$,
    \begin{align*}
        2C\left\|\hat P_k(\cdot|s,a) - P(\cdot|s,a)\right\|_1 \leq 2CS\sqrt{\frac{\log(2KS^2|\calA|/\delta)}{2(N_k(s,a)\vee 1)}}, \ \ \forall (k,s,a) \in [K]\times\calS\times\calA
    \end{align*}
    By applying this to \eqref{eq:concen 1}, with probability at least $1-\delta$, for $(k,s,a) \in [K]\times\calS\times\calA$, we have
    \begin{align*}
        \left|\left(\hat P_k(\cdot|s,a) - P(\cdot|s,a)\right)^\top V_{k,h+1}(\cdot, \psi(s,a,z))\right|\leq 2CS\sqrt{\frac{\log(2KS^2|\calA|/\delta)}{2(N_k(s,a)\vee 1)}}
    \end{align*}
    as desired.
\end{proof}

\begin{lemma}[Restatement of Lemma~\ref{lem:optimism}]\label{lem:optimism app}
    Suppose that \Cref{lem:concentration} holds. For any $(k,h,s,z) \in [K]\times [H]\times \calS \times\calZ$, we have
    \begin{align*}
        V_{\tilde r, h}^{\pi^*}(s,z) \leq V_{k,h}(s,z).
    \end{align*}
\end{lemma}
\begin{proof}
    Fix $k \in [K]$. We prove the statement by induction on $h = H+1, \ldots, 1$. For the base case, since $V_{k,H+1}(s,z)$ is initialized to be $0$, we have $V_{\tilde r, H+1}^{\pi^*}(s,z)= V_{k,H+1}(s,z) =0$ for all $(s,z) \in \calS \times\calZ$. Now, we assume that $V_{\tilde r, h+1}^{\pi^*}(s,z) \leq V_{k,h+1}(s,z)$ for all $(s,z) \in \calS \times\calZ$. For simplicity, define the auxiliary function $\psi$ such that
    \[
        \psi(s,a,z) = \Pi_\calZ(z-g(s,a)).
    \]
    Note that
    \begin{align*}
        Q_{\tilde r, h}^{\pi^*}(s,a,z) 
        &= \tilde r(s,a,z) + \sum_{s'\in\calS, z'\in\bbR}\tilde P(s',z'|s,a,z) V_{\tilde r,h+1}^{\pi^*}(s',z')\\
        &= \tilde r(s,a,z) + \sum_{s'\in\calS} P(s'|s,a)V_{\tilde r, h+1}^{\pi^*}(s',\psi(s,a,z))
    \end{align*}
    where the first equality follows from the definition of the $Q$-function, and the last equality follows from the definition of $\tilde P(s',z'|s,a,z) = \Ind\{z' = \Pi_\calZ(z-g(s,a))\}P(s'|s,a) = \Ind\{z' = \psi(s,a,z)\}P(s'|s,a)$. Then it follows that
    \begin{align*}
        & Q_{k,h}(s,a,z) - Q_{\tilde r,h}^{\pi^*}(s,a,z) \\
        &= \hat P_k(\cdot|s,a)^\top V_{k,h+1}(\cdot, \psi(s,a,z)) - P(\cdot|s,a)^\top V_{\tilde r,h+1}^{\pi^*}(s',\psi(s,a,z)) + \frac{\beta}{\sqrt{N_k(s,a)}}\\
        &= \left(\hat P_k(\cdot|s,a) - P(\cdot|s,a)\right)^\top V_{k,h+1}(\cdot, \psi(s,a,z)) \\
        &\quad + P(\cdot|s,a)^\top \left(V_{k,h+1}(\cdot,\psi(s,a,z)) - V_{\tilde r,h+1}^{\pi^*}(\cdot,\psi(s,a,z))\right) + \frac{\beta}{\sqrt{N_k(s,a)}} \\
        &\geq 0
    \end{align*}
    where the inequality follows from \Cref{lem:concentration}, i.e., $(\hat P_k(\cdot|s,a) - P(\cdot|s,a))^\top V_{k,h+1}(\cdot, \psi(s,a,z)) \geq -\beta/\sqrt{N_k(s,a)}$, and the induction hypothesis, i.e., $V_{k,h+1}(s,z) \geq V_{\tilde r, h+1}^{\pi^*}(s,z)$. Then we have for all $(s,z) \in \calS \times \calZ$,
    \begin{align}\label{eq:tilde V > V}
        \widetilde V_{k,h}(s,z)=\max_{a\in\calA} Q_{k,h}(s,a,z) \geq \sum_{a\in\calA}\pi^*(a|s)Q_{\tilde r,h}^{\pi^*}(s,a,z) = V_{\tilde r, h}^{\pi^*}(s,z).
    \end{align}
    Recall that $V_{k,h}(s,z) = \widetilde V_{k,h}(s,z) \wedge (\min_{s'}\widetilde V_{k,h}(s',z) + 2C)$ for all $z\in\calZ$ by the algorithm. Then it follows that for all $(s,z) \in \calS \times\calZ$,
    \begin{align*}
        V_{k,h}(s,z) 
        &= \widetilde V_{k,h}(s,z) \wedge (\min_{s'}\widetilde V_{k,h}(s',z) + 2C) \\
        &\geq V_{\tilde r,h}^{\pi^*}(s,z) \wedge (\min_{s'} V_{\tilde r,h}^{\pi^*}(s',z) + 2C) \\
        &= V_{\tilde r,h}^{\pi^*}(s,z)
    \end{align*}
    where the inequality follows from \eqref{eq:tilde V > V}, and the last equality follows from the first statement of \Cref{lem:span app}. This concludes the induction, implying that $V_{k,h}(s,z) \geq V_{\tilde r, h}^{\pi^*}(s,z)$ for all $(h,s,z)\in[H] \times \calS \times\calZ$. Since we can apply the same argument for all $k$, we conclude the proof.
\end{proof}

\begin{lemma}\label{lem:martingale sum}
    With probability at least $1-\delta$, we have
    \begin{align*}
    \sum_{k=1}^K\sum_{h=1}^H \left(P(\cdot|s_h^k,a_h^k)^\top V_{k,h+1}(\cdot, z_{h+1}^k) - V_{k,h+1}(s_{h+1}^k,z_{h+1}^k)\right) \leq 2C\sqrt{2T\log(2/\delta)}.
    \end{align*}
\end{lemma}
\begin{proof}
    Recall that $\calF_{k,h} = \sigma(s_1,a_1,\ldots,s_h^k, a_h^k)$. Note that $z_{h+1}^k$ is $\calF_{k,h}$-measurable, and $V_{k,h+1}$ is $\calF_{k-1,H}$-measurable. Then it follows that
    \begin{align*}
        P(\cdot|s_h^k,a_h^k)^\top V_{k,h+1}(\cdot,z_{h+1}^k) = \bbE\left[V_{k,h+1}(s_{h+1}^k,z_{h+1}^k)|\calF_{k,h}\right].
    \end{align*}
    Define $X_{k,h} = P(\cdot|s_h^k,a_h^k)^\top V_{k,h+1}(\cdot, z_{h+1}^k) - V_{k,h+1}(s_{h+1}^k,z_{h+1}^k)$, which satisfies $\bbE[X_{k,h}|\calF_{k,h}] = 0$. Thus, $\{X_{k,h}\}_{k\in [K],h\in[H]}$ is a martingale difference sequence. Then by the Azuma-Hoeffding inequality, with probability at least $1-\delta$, we have
    \begin{align*}
        \sum_{k=1}^K \sum_{h=1}^H X_{k,h} \leq 2C\sqrt{2T\log(2/\delta)}.
    \end{align*}
\end{proof}

\begin{lemma}\label{lem:sum N}
    We have
    \begin{align*}
        \sum_{k=1}^K\sum_{h=1}^H\frac{2\beta}{\sqrt{N_k(s_h^k,a_h^k) \vee 1}} \leq 4\beta HSA + 6\beta\sqrt{SAT}.
    \end{align*}
\end{lemma}
\begin{proof}
    Note that
    \begin{align*}
        \sum_{k=1}^K\sum_{h=1}^H\frac{1}{\sqrt{N_k(s_h^k,a_h^k) \vee 1}} 
        &= \sum_{k=1}^K\sum_{h=1}^H\frac{1}{\sqrt{N_k(s_h^k,a_h^k) \vee 1}}\\
        &=  \sum_{s\in\calS}\sum_{a\in\calA}\sum_{k=1}^K\frac{\sum_{h=1}^H \Ind\{s_h^k=s,a_h^k=a\}}{\sqrt{N_k(s,a)\vee 1}}.
    \end{align*}
    Note that
    \begin{align}\label{eq:sum N 1}
    \begin{aligned}
        &\sum_{k=1}^K\frac{\sum_{h=1}^H \Ind\{s_h^k=s,a_h^k=a\}}{\sqrt{N_k(s,a)\vee 1}} \\
        &= \sum_{k:N_k(s,a)<H}\frac{\sum_{h=1}^H \Ind\{s_h^k=s,a_h^k=a\}}{\sqrt{N_k(s,a)\vee 1}} + \sum_{k:N_k(s,a)\geq H}\frac{\sum_{h=1}^H \Ind\{s_h^k=s,a_h^k=a\}}{\sqrt{N_k(s,a)\vee 1}}\\
        &= \sum_{k:N_k(s,a)<H}\frac{\sum_{h=1}^H \Ind\{s_h^k=s,a_h^k=a\}}{\sqrt{N_k(s,a)\vee 1}} + \sum_{k:N_k(s,a)\geq H}\frac{\sum_{h=1}^H \Ind\{s_h^k=s,a_h^k=a\}}{\sqrt{N_k(s,a)\vee 1}}.
    \end{aligned}
    \end{align}
    We bound the first term. For simplicity, let $k_0 = \max\{k\in[K]: N_k(s,a) < H\}$. Recall that $N_k(s,a)$ denotes the number of visits to $(s,a)$ up to $k-1$ episode. Then, we know that
    \begin{align}\label{eq:sum N 2}
    \begin{aligned}
        \sum_{k:N_k(s,a)<H}\frac{\sum_{h=1}^H \Ind\{s_h^k=s,a_h^k=a\}}{\sqrt{N_k(s,a)\vee 1}}
        &\leq \sum_{k:N_k(s,a)<H}\sum_{h=1}^H\Ind\{s_h^k=s,a_h^k=a\} \\
        &= \sum_{k=1}^{k_0}\sum_{h=1}^H \Ind\{s_h^k=s,a_h^k=a\}\\
        &= \sum_{k=1}^{k_0-1}\sum_{h=1}^H \Ind\{s_h^k=s,a_h^k=a\} + \sum_{h=1}^H\Ind\{s_h^{k_0}=s,a_h^{k_0}=a\}\\
        &\leq N_{k_0}(s,a) + H \\
        &\leq 2H
    \end{aligned}
    \end{align}
    where the first inequality follows from $\sqrt{N_k(s,a)\vee 1} \geq 1$, the first equality follows from the definition of $k_0$, the second inequality follows from the definition of $N_{k_0}(s,a)$ and the fact that $\sum_{h=1}^H \Ind\{s_h^{k_0}=s,a_h^{k_0}=a\} \leq H$, and the last inequality is again due to the definition of $k_0$.

    We then bound the second term. For each $k$ such that $N_k(s,a)\geq H$, we have
    \begin{align}\label{eq:Nk+1 < 2Nk}
    \begin{aligned}
        N_{k+1}(s,a) &= N_k(s,a) + \sum_{h=1}^H \Ind\{s_h^k=s,a_h^k=a\}\\
        &\leq N_k(s,a) + H \\
        &\leq 2N_k(s,a).
    \end{aligned}
    \end{align}
    Then it follows that
    \begin{align}\label{eq:sum N 3}
    \begin{aligned}
        \sum_{k:N_k(s,a)\geq H}\frac{\sum_{h=1}^H \Ind\{s_h^k=s,a_h^k=a\}}{\sqrt{N_k(s,a)\vee 1}} 
        &= \sum_{k:N_k(s,a)\geq H}\frac{\sum_{h=1}^H \Ind\{s_h^k=s,a_h^k=a\}}{\sqrt{N_k(s,a)}}\\
        &= 3\sum_{k:N_k(s,a)\geq H}\frac{\sum_{h=1}^H \Ind\{s_h^k=s,a_h^k=a\}}{\sqrt{4N_k(s,a)} + \sqrt{N_k(s,a)}}\\
        &\leq 3\sum_{k:N_k(s,a)\geq H}\frac{\sum_{h=1}^H \Ind\{s_h^k=s,a_h^k=a\}}{\sqrt{N_{k+1}(s,a)} + \sqrt{N_k(s,a)}}\\
        &= 3\sum_{k:N_k(s,a)\geq H}\left(\sqrt{N_{k+1}(s,a)}-\sqrt{N_k(s,a)}\right)\\
        &\leq 3\sqrt{N_{K+1}(s,a)}
    \end{aligned}
    \end{align}
    where the inequality follows from \eqref{eq:Nk+1 < 2Nk}, i.e., $4N_k(s,a) \geq 2N_k(s,a) \geq N_{k+1}(s,a)$, and the third equality follows from the fact that $\sum_{h=1}^H \Ind\{s_h^k=s,a_h^k=a\} = N_{k+1}(s,a) -  N_k(s,a) = (\sqrt{N_{k+1}(s,a)} - \sqrt{N_{k}(s,a)})(\sqrt{N_{k+1}(s,a)} + \sqrt{N_{k}(s,a)})$. By applying \eqref{eq:sum N 2} and \eqref{eq:sum N 3} to \eqref{eq:sum N 1}, we have
    \begin{align*}
        \sum_{k=1}^K\frac{\sum_{h=1}^H \Ind\{s_h^k=s,a_h^k=a\}}{\sqrt{N_k(s,a)\vee 1}} 
        &\leq 2H + 3\sqrt{N_{K+1}(s,a)}.
    \end{align*}
    By summing this over $(s,a)\in\calS\times\calA$, we have
    \begin{align*}
        \sum_{s\in\calS}\sum_{a\in\calA}\sum_{k=1}^K\frac{\sum_{h=1}^H \Ind\{s_h^k=s,a_h^k=a\}}{\sqrt{N_k(s,a)\vee 1}} 
        &\leq 2HSA + 3\sum_{s,a}\sqrt{N_{K+1}(s,a)}\\
        &\leq 2HSA + 3\sqrt{SA}\sqrt{\sum_{s,a}N_{K+1}(s,a)} \\
        &\leq 2HSA + 3\sqrt{SAT}.
    \end{align*}
    Finally, we have shown that
    \begin{align*}
        \sum_{k=1}^K\sum_{h=1}^H\frac{2\beta}{\sqrt{N_k(s_h^k,a_h^k) \vee 1}} 
        &\leq 2\beta \sum_{s\in\calS}\sum_{a\in\calA}\sum_{k=1}^K\frac{\sum_{h=1}^H \Ind\{s_h^k=s,a_h^k=a\}}{\sqrt{N_k(s,a)\vee 1}}\\
        &\leq 4\beta HSA + 6\beta\sqrt{SAT}.
    \end{align*}
    This concludes the proof.
\end{proof}

\begin{lemma}[Formal statement of \eqref{eq:RegAug}]\label{lem:tilde Regret}
Suppose that Lemmas~\ref{lem:concentration} and~\ref{lem:martingale sum} hold. Then we have
    \begin{align*}
        \RegAug(T) 
        \leq 2C\sqrt{2T\log(2/\delta)} + 4\beta HSA + 6\beta\sqrt{SAT}.
    \end{align*}
\end{lemma}
\begin{proof}
    For each $k\in [K]$, we have
\begin{align*}
    V_{\tilde r,1}^{\pi^*}(s_1^k, z_1^k) - \sum_{h=1}^H \tilde r(s_h^k,a_h^k,z_h^k) 
    &= \underbrace{V_{\tilde r,1}^{\pi^*}(s_1^k,z_1^k) - V_{k,1}(s_1^k,z_1^k)}_{\text{(I)}} \\
    &\quad+ \underbrace{V_{k,1}(s_1^k,z_1^k)- \sum_{h=1}^H \tilde r(s_h^k,a_h^k,z_h^k)}_{\text{(II)}}.
\end{align*}
Note that $\text{(I)}\leq 0$ by \Cref{lem:optimism app}, and (II) can be rewritten as follows.
\begin{align*}
    V_{k,1}(s_1^k,z_1^k)- \sum_{h=1}^H \tilde r(s_h^k,a_h^k,z_h^k)
    &= V_{k,1}(s_1^k,z_1^k) - \tilde V_{k,1}(s_1^k,z_1^k) + \tilde V_{k,1}(s_1^k,z_1^k) - Q_{k,1}(s_1^k,a_1^k,z_1^k) \\
    &\quad+ \left(\hat P_k(\cdot|s_1^k,a_1^k)-P(\cdot|s_1^k,a_1^k)\right)^\top V_{k,2}(\cdot, z_{2}^k) \\
    &\quad+ P(\cdot|s_1^k,a_1^k)^\top V_{k,2}(\cdot, z_{2}^k) - V_{k,2}(s_2^k,z_2^k)\\
    &\quad+ \beta/\sqrt{N_k(s_1^k,a_1^k) \vee 1} \\
    &\quad+V_{k,2}(s_2^k,z_2^k)- \sum_{h=2}^H \tilde r(s_h^k,a_h^k,z_h^k).
\end{align*}
Each term can be bounded as follows. Due to the fact that $z_1^k \in \calZ$ (\Cref{lem:z in Z}) and the design of the clipping operation, we have $V_{k,1}(s_1^k,z_1^k)-\tilde V_{k,1}(s_1^k,z_1^k) \leq 0$. Since $\tilde V_{k,1}(s_1^k,z_1^k) = \max_{a} Q_{k,1}(s_1^k,a,z_1^k) = Q_{k,1}(s_1^k,a_1^k,z_1^k)$, we have $\tilde V_{k,1}(s_1^k,z_1^k)-Q_{k,1}(s_1^k,a_1^k,z_1^k)=0$. By \Cref{lem:concentration}, we have $(\hat P_k(\cdot|s_1^k,a_1^k)-P(\cdot|s_1^k,a_1^k))^\top V_{k,2}(\cdot, z_{2}^k) \leq \beta/\sqrt{N_k(s_1^k,a_1^k)\vee1}$. Applying these yields
\begin{align*}
    V_{k,1}(s_1^k,z_1^k)- \sum_{h=1}^H \tilde r(s_h^k,a_h^k,z_h^k)
    &\leq P(\cdot|s_1^k,a_1^k)^\top V_{k,2}(\cdot, z_{2}^k) - V_{k,2}(s_2^k,z_2^k)\\
    &\quad+ \frac{2\beta}{\sqrt{N_k(s_1^k,a_1^k) \vee 1}} + V_{k,2}(s_2^k,z_2^k)- \sum_{h=2}^H \tilde r(s_h^k,a_h^k,z_h^k).
\end{align*}
Due to the recursion, we have
\begin{align*}
    \text{(II)}&= V_{k,1}(s_1^k,z_1^k)- \sum_{h=1}^H \tilde r(s_h^k,a_h^k,z_h^k) \\
    &\leq \sum_{h=1}^H \left(P(\cdot|s_h^k,a_h^k)^\top V_{k,h+1}(\cdot, z_{h+1}^k) - V_{k,h+1}(s_{h+1}^k,z_{h+1}^k)\right)\\
    &\quad+ \sum_{h=1}^H\frac{2\beta}{\sqrt{N_k(s_h^k,a_h^k) \vee 1}}.
\end{align*}
By applying bounds on (I) and (II), we have
\begin{align*}
    V_{\tilde r,1}^{\pi^*}(s_1^k, z_1^k) - \sum_{h=1}^H \tilde r(s_h^k,a_h^k,z_h^k)
    &\leq \sum_{h=1}^H \left(P(\cdot|s_h^k,a_h^k)^\top V_{k,h+1}(\cdot, z_{h+1}^k) - V_{k,h+1}(s_{h+1}^k,z_{h+1}^k)\right)\\
    &\quad+ \sum_{h=1}^H\frac{2\beta}{\sqrt{N_k(s_h^k,a_h^k) \vee 1}}.
\end{align*}
By summing this over $k=1,\ldots,K$, we have
\begin{align*}
    &\sum_{k=1}^K\left(V_{\tilde r,1}^{\pi^*}(s_1^k, z_1^k) - \sum_{h=1}^H \tilde r(s_h^k,a_h^k,z_h^k)\right)\\
    &\leq \sum_{k=1}^K\sum_{h=1}^H \left(P(\cdot|s_h^k,a_h^k)^\top V_{k,h+1}(\cdot, z_{h+1}^k) - V_{k,h+1}(s_{h+1}^k,z_{h+1}^k)\right) + \sum_{k=1}^K\sum_{h=1}^H\frac{2\beta}{\sqrt{N_k(s_h^k,a_h^k) \vee 1}}\\
    &\leq 2C\sqrt{2T\log(2/\delta)} + 4\beta HSA + 6\beta\sqrt{SAT}.
\end{align*}
\end{proof}

\subsection{Deferred proofs for Step 2 of Section~\ref{sec:analysis}}
\begin{lemma}[Formal statement of \eqref{eq:Regret + PhiW main}]\label{lem:Regret + PhiW}
    Suppose that Lemmas~\ref{lem:concentration} and~\ref{lem:martingale sum} hold. Then we have
    \begin{align*}
        \Regret(T) + \Phi_W(z_{T+1}) \leq 2C\sqrt{2T\log(2/\delta)} + 4\beta HSA + 6\beta\sqrt{SAT} + KC.
    \end{align*}
\end{lemma}
\begin{proof}
    Recall that $\tilde r(s,a,z) = r(s,a) + \Phi_W(z) - \Phi_W(\psi(s,a,z)$ and $\Phi_W(z_1^1) = 0$. Then we have
    \begin{align*}
        -\sum_{t=1}^T \tilde r(s_t,a_t, z_t) 
        &= -\sum_{t=1}^T \left(r(s_t,a_t) + \Phi_W(z_t) - \Phi_W(\psi(s_t,a_t,z_t))\right)\\
        &= -\sum_{t=1}^T \left(r(s_t,a_t) + \Phi_W(z_t) - \Phi_W(z_{t+1})\right)\\
        &= -\sum_{t=1}^T r(s_t,a_t) + \Phi_W(z_{T+1})
    \end{align*}
    Then we have
    \begin{align*}
        \RegAug(T) 
        &= \sum_{k=1}^K \left(V_{\tilde r,1}^{\pi^*}(s_1^k,z_1^k) - \sum_{h=1}^H\tilde r(s_h^k,a_h^k,z_h^k)\right) \\
        &= TJ_r^* - \sum_{t=1}^T r(s_t,a_t) + \Phi_W(z_{T+1}) + \sum_{k=1}^K V_{\tilde r,1}^{\pi^*}(s_1^k,a_1^k) - TJ_r^* \\
        &= \Regret(T) + \Phi_W(z_{T+1}) + \sum_{k=1}^K V_{\tilde r,1}^{\pi^*}(s_1^k,a_1^k) - TJ_r^*
    \end{align*}
    Then it can be rewritten as
    \begin{align*}
        \Regret(T) + \Phi_W(z_{T+1})= \RegAug(T) + \sum_{k=1}^K \left(HJ_r^* -  V_{\tilde r,1}^{\pi^*}(s_1^k,a_1^k)\right).
    \end{align*}
    By \Cref{lem:tilde Regret}, we have
    \begin{align*}
        \RegAug(T) \leq 2C\sqrt{2T\log(2/\delta)} + 4\beta HSA + 6\beta\sqrt{SAT}. 
    \end{align*}
    By \Cref{lem:span app}, we have
    \begin{align*}
        \sum_{k=1}^K \left(HJ_r^* -  V_{\tilde r,1}^{\pi^*}(s_1^k,a_1^k)\right) \leq KC.
    \end{align*}
    Finally, we have
    \begin{align*}
        \Regret(T) + \Phi_W(z_{T+1})\leq 2C\sqrt{2T\log(2/\delta)} + 4\beta HSA + 6\beta\sqrt{SAT} + KC.
    \end{align*}
\end{proof}

\subsection{Deferred proofs for Step 3 of Section~\ref{sec:analysis}}
\begin{lemma}[Formal statement of \eqref{eq:zT main}]\label{lem:ZT+1}
    Suppose that Lemmas~\ref{lem:regret lower bound},~\ref{lem:concentration}, and~\ref{lem:martingale sum} hold. Then we have
    \begin{align*}
        z_{T+1} \leq \tbigO\left((\spn(v_r^*) +\spn(v_g^*)+\spn(v_g^*)^2)S^2A\sqrt{T}\right)
    \end{align*}
\end{lemma}
\begin{proof}
    Recall that, by Lemmas~\ref{lem:Regret + PhiW} and~\ref{lem:regret lower bound}, we have
    \begin{align*}
        &\Regret(T) + \Phi_W(z_{T+1}) \leq 2C\sqrt{2T\log(2/\delta)} + 4\beta HSA + 6\beta\sqrt{SAT} + KC, \\
        &\Regret(T) \geq -\lambda^* \Violation(T) - \spn({v_{\lambda^*}^*})\sqrt{2T\log(1/\delta)} - \spn({v_{\lambda^*}^*}).
    \end{align*}
    These lead to
    \begin{align*}
        \Phi_W(z_{T+1}) 
        &\leq 2C\sqrt{2T\log(2/\delta)} + 4\beta HSA + 6\beta\sqrt{SAT} + KC \\
        &\quad+ \lambda^* \Violation(T) + \spn({v_{\lambda^*}^*})\sqrt{2T\log(1/\delta)} + \spn({v_{\lambda^*}^*})
    \end{align*}
    Moreover, by applying a similar argument in the proof of \Cref{lem:z in Z}, we can show the following.
    \begin{align}\label{eq:Viol - Z}
    \begin{aligned}
        \left|\Violation(T) - z_{T+1}\right| 
        &= \left|\sum_{t=1}^T (0-g(s_t,a_t))) - z_{T+1}\right|\\
        &= \left|-\sum_{t=1}^T g(s_t,a_t)-\sum_{t=1}^T (z_{t+1} - z_t)\right|\\
        &=\left|\sum_{t=1}^T \left(z_t - g(s_t,a_t) - z_{t+1}\right)\right|\\
        &\leq\sum_{t=1}^T \left|z_t - g(s_t,a_t) - z_{t+1}\right|\\
        &\leq \frac{\Delta T}{2}
    \end{aligned}
    \end{align}
    Thus, we have
    \begin{align*}
    \begin{aligned}
        \Phi_W(z_{T+1}) 
        &\leq 2C\sqrt{2T\log(2/\delta)} + 4\beta HSA + 6\beta\sqrt{SAT} + KC  \\
        &\quad+ \lambda^* z_{T+1} + \lambda^* \left|\Violation(T) - z_{T+1}\right|\\
        &\quad+\spn(v_{\lambda^*}^*)\sqrt{2T\log(1/\delta)}+ \spn(v_{\lambda^*}^*)\\
        &\leq 2C\sqrt{2T\log(2/\delta)} + 4\beta HSA + 6\beta\sqrt{SAT} + KC  \\
        &\quad+ \lambda^* z_{T+1} + \frac{\lambda^*\Delta T}{2}+\spn(v_{\lambda^*}^*)\sqrt{2T\log(1/\delta)}+ \spn(v_{\lambda^*}^*).
    \end{aligned}
    \end{align*}
    Then it can be rewritten as
    \begin{align*}
        \Phi_W(z_{T+1}) - \lambda^*z_{T+1} 
        &\leq 2C\sqrt{2T\log(2/\delta)} + 4\beta HSA + 6\beta\sqrt{SAT} + KC \\
        &\quad + \frac{\lambda^*\Delta T}{2}+\spn(v_{\lambda^*}^*)\sqrt{2T\log(1/\delta)}+ \spn(v_{\lambda^*}^*).
    \end{align*}
    By the fourth statement of \Cref{lem:PhiW app}, we have $(\Lambda - \lambda^*)z_{T+1} - \Lambda W / 2 \leq \Phi_W(z_{T+1}) - \lambda^*z_{T+1}$. Then it follows that
    \begin{align*}
        (\Lambda - \lambda^*)z_{T+1} - \Lambda W / 2 
        &\leq 2C\sqrt{2T\log(2/\delta)} + 4\beta HSA + 6\beta\sqrt{SAT} + KC \\
        &\quad + \frac{\lambda^*\Delta T}{2}+\spn(v_{\lambda^*}^*)\sqrt{2T\log(1/\delta)}+ \spn(v_{\lambda^*}^*).
    \end{align*}
    Recall that $\Lambda = 2/\gamma$ and $\lambda^* \leq 1/\gamma$ by \Cref{lem:optimal dual variable bound}. Thus, we have $\Lambda - \lambda^* \geq 1/\gamma$. Then it follows that
    \begin{align*}
        z_{T+1} &\leq 2\gamma C\sqrt{2T\log(2/\delta)} + 4\gamma \beta HSA + 6\gamma \beta\sqrt{SAT} + \gamma KC \\
        &\quad+ \frac{\gamma \lambda^*\Delta T}{2}+\spn(v_{\lambda^*}^*)\sqrt{2T\log(1/\delta)}+ \gamma \spn(v_{\lambda^*}^*) + \frac{\gamma \Lambda W}{2}.
    \end{align*}
    Again, recall that $\lambda^*\leq 1/\gamma$, $0 \leq \gamma \leq 1$, and the algorithmic parameter choice~\eqref{eq:alg param}, i.e.,
    \begin{align*}
    \Delta &= \frac{1}{\sqrt{T}},\ W=K=H = \sqrt{T}, \\
    C &= \spn(v_r^*) + \Lambda \spn(v_g^*) + \frac{\Lambda}{W}\left(\spn(v_g^*)^2 + 2H(\spn(v_g^*)+2)^2\right)+ \frac{H\Lambda\Delta}{2},\\
    &=\bigO\left(\spn(v_r^*) + \frac{1}{\gamma} (\spn(v_g^*) + \spn(v_g^*)^2) + \frac{1}{\gamma}\right),\\
    \beta &= 2CS\sqrt{\log(2TS^2|\calA|/\delta)/2} \\
    &\leq \tbigO\left(S\spn(v_r^*) + \frac{S}{\gamma} (\spn(v_g^*) + \spn(v_g^*)^2) + \frac{S}{\gamma}\right). 
\end{align*}
Then it follows that
\begin{align*}
    z_{T+1} \leq \tbigO\left((\spn(v_r^*) +\spn(v_g^*)+\spn(v_g^*)^2)S^2A\sqrt{T} + \spn(v_{\lambda^*}^*)\sqrt{T}\right)
\end{align*}
\end{proof}

\subsection{Proof of Theorem~\ref{thm:main}}
    By a union bound, Lemmas~\ref{lem:regret lower bound},~\ref{lem:concentration}, and~\ref{lem:martingale sum} hold with probability at least $1-3\delta$. Throughout the proof, suppose that these lemmas hold. Then, by \Cref{lem:Regret + PhiW}, we have
    \begin{align}\label{eq:Regret + PhiW}
        \Regret(T) + \Phi_W(z_{T+1}) \leq 2C\sqrt{2T\log(2/\delta)} + 4\beta HSA + 6\beta\sqrt{SAT} + KC. 
    \end{align}
    For $\Regret(T)$, since $\Phi_W(z_{T+1})\geq 0$, it directly follows that
    \begin{align*}
        \Regret(T) 
        &\leq 2C\sqrt{2T\log(2/\delta)} + 4\beta HSA + 6\beta\sqrt{SAT} + KC\\
        &\leq \tbigO\left(\left(\spn(v_r^*)+\frac{1}{\gamma}(\spn(v_g^*)+\spn(v_g^*)^2+1)\right)S^2A\sqrt{T}\right)
    \end{align*}
    Next, we bound $\Violation(T)$. By \Cref{lem:ZT+1}, we have
    \begin{align*}
        z_{T+1} \leq \tbigO\left((\spn(v_r^*) +\spn(v_g^*)+\spn(v_g^*)^2)S^2A\sqrt{T}+\spn(v_{\lambda^*}^*)\sqrt{T}\right).
    \end{align*}
    Moreover, by \eqref{eq:Viol - Z}, we have $|\Violation(T) - z_{T+1}| \leq \Delta T/2 \leq \tbigO(\sqrt{T})$. Then it follows that
    \begin{align*}
        \Violation(T) 
        &\leq z_{T+1} + |\Violation(T) - z_{T+1}|\\
        &\leq \tbigO\left((\spn(v_r^*) +\spn(v_g^*)+\spn(v_g^*)^2)S^2A\sqrt{T}+\spn(v_{\lambda^*}^*)\sqrt{T}\right).
    \end{align*}
    \qed

\section{Auxiliary Lemmas}

\begin{lemma}\label{lem:(p-q)v}
    Let $d\in \bbZ_+$, and let $p,q \in [0,1]^d$ such that $\sum_{i=1}^d p_i = \sum_{i=1}^d q_i = 1$. Let $v \in \bbR^d$. Then we have $|(p-q)^\top v| \leq \spn(v) ||p-q||_1$.
\end{lemma}
\begin{proof}
    Let $v_{i_{\min}} = \argmin_{j\in[d]} v_j$. Note that
    \begin{align*}
        |(p-q)^\top v| 
        &= \left|\sum_{i=1}^d (p_i-q_i)v_i\right| \\
        &= \left|\sum_{i=1}^d (p_i-q_i)(v_i-v_{i_{\min}})\right|\\
        &\leq ||p-q||_1\cdot\max_{j\in [d]} |v_j - v_{i_{\min}}|\\
        &= ||p-q||_1 \cdot \spn(v)
    \end{align*}
    where the second equality follows from the fact that $\sum_{i=1}^d(p_i-q_i) v_{i_{\min}} = 0$, the inequality follows from the Cauchy-Schwarz inequality, and the last equality follows from $\max_{j\in [d]} |v_j - v_{i_{\min}}| = \spn(v)$.

\end{proof}

\begin{lemma}
[Lemma 4 in \citet{chen2022learning}]
\label{lem:chen2022 lemma 4}
    For any stationary policy $\pi$ and reward function $r \in \bbR^{S \times|\calA|}$ such that $J_r^{\pi}(s) = J_r^\pi$ for all $s\in\calS$, we have $|V_{r,h}^{\pi}(s) - (H-h+1)J_r^\pi| \leq \spn(v_r^\pi)$ for $(s,h)\in \calS \times [H]$.
\end{lemma}

\section{Experimental Details}\label{app:experiments}

\paragraph{Environment}
We use a tabular CMDP constructed by modifying the MDP of \citet{he2022near} to incorporate a constraint and an infinite-horizon setting, following \citet{yu2026learning}. The state space is $\calS = \{0, \ldots, H_{\mathrm{env}}+1\}$, and the action space is $\calA = \{-1, 1\}^{m-1}$. For each state $s \in \{0, \ldots, H_{\mathrm{env}}-1\}$ and action $a \in \calA$, the constraint value is $g(s,a) = \mathrm{score}(a)$, where $\mathrm{score}(a)$ denotes the fraction of positive components in $a$. The process transitions to state $s+1$ with probability $0.95 - 0.01 \cdot \mathbf{1}_{m-1}^\top a$ and to state $H_{\mathrm{env}}+1$ otherwise. State $H_{\mathrm{env}}$ yields reward $0$ and state $H_{\mathrm{env}}+1$ yields reward $1$. Both states incur zero constraint value and transition back to state $0$, which makes the process infinite-horizon. The reward for $s \in \{0, \ldots, H_{\mathrm{env}}-1\}$ is $r(s,a) = 0.4 \cdot (1 - \mathrm{score}(a))$ for $s < H_{\mathrm{env}}/2$ and $r(s,a) = 0.4 \cdot \mathrm{score}(a)$ for $s \geq H_{\mathrm{env}}/2$, so in the early states the reward is given to the action that lowers the constraint value, whereas the CMDP of \citet{yu2026learning} uses $r(s,a) = 0.4 \cdot \mathrm{score}(a)$ in every state. We choose $H_{\mathrm{env}} = 6$, $m = 2$, $b = 0.7$, and $s_1 = 0$. We solve linear programs over occupancy measures to compute the optimal policy, $J_r^*$, and the Slater constant $\gamma = \max_\pi J_g^\pi - b$. We then compute the bias spans of the optimal policy.

\paragraph{Metrics and hyperparameters}
We report regret and constraint violation as
\begin{align*}
\Regret(t) &= \sum_{\tau=1}^{t}(J_r^* - r(s_\tau,a_\tau)), &
\Violation(t) &= \left[\sum_{\tau=1}^{t}(b - g(s_\tau,a_\tau))\right]_+.
\end{align*}
\Cref{alg:main} is run with the shifted constraint function $g - b$. For \Cref{alg:main}, we set $\Lambda = 2/\gamma$ and $H = W = \lfloor\sqrt{T}\rfloor$ as in \Cref{thm:main}, and $\Delta = 0.1/\lceil 0.1\sqrt{T}\rceil \leq 1/\sqrt{T}$, for which the values of $g - b$ are integer multiples of $\Delta$ so that the projection $\Pi_\Delta$ has no rounding error. Following the adaptive clipping parameter $\zeta_k = \spn(v_r^*) + \lambda_k\spn(v_g^*)$ of \citet{yu2026learning}, we use the clipping parameter $C_k = \spn(v_r^*) + \Phi_W'(z_1^k)\spn(v_g^*)$, which replaces $\Lambda$ in $C$ by the slope of $\Phi_W$ at $z_1^k$, and set the bonus term to $\beta_k = C_k|\calS|$, omitting the constant and the logarithmic factor. For the baselines, we use one-hot features with $d = |\calS||\calA| = 16$ and follow the parameter choices of \citet{yu2026learning}. That is, (i) for PD-LSCVI-UCB, we set $H = T^{1/3}$, the clipping parameter $\zeta_k$, and $\beta_k = (\zeta_k+1)d$, and (ii) for Algorithm 2 of \citet{ghosh2023achieving}, we follow its default parameters except that $H = \lceil T^{1/4}\rceil$ and $\beta = dH$. (iii) For Algorithm 3 of \citet{chen2022learning}, we set $H = \lceil (T/|\calS|^2|\calA|)^{1/3} \rceil = 7$ and solve its linear program over occupancy measures in every episode. Its constraint $\langle \nu, c \rangle \leq H\tau + \spn(v_c^*)$ is written with the cost $c = 1 - g$ and the threshold $\tau = 1 - b$, which is equivalent to $\langle \nu, g \rangle \geq Hb - \spn(v_g^*)$. For its Bernstein-type confidence set, we use the radius $\sqrt{\bar{P}_k(s' \mid s,a)/N_k(s,a)} + 1/N_k(s,a)$ around the empirical transition $\bar{P}_k$, omitting the constants and the logarithmic factor, and the sensitivity experiment below multiplies this radius by $c$.

\begin{figure}[t]
    \centering
    \includegraphics[width=\linewidth]{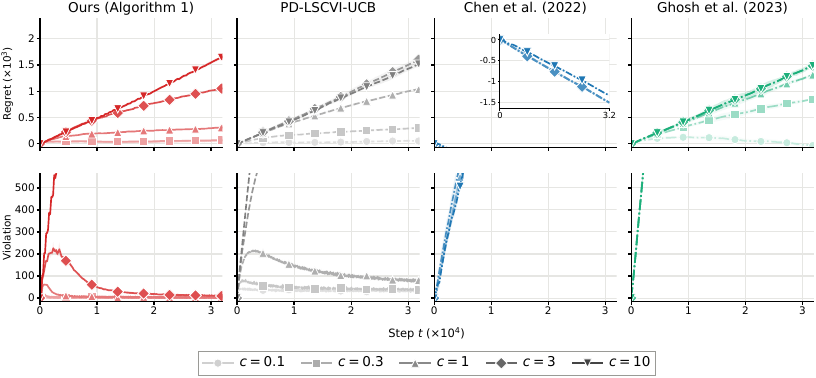}
    \vspace{-14pt}
    \caption{Plots of regret and constraint violation when the bonus term of each algorithm is multiplied by $c \in \{0.1, 0.3, 1, 3, 10\}$, where lighter curves correspond to smaller $c$. Each plot represents the average over $3$ trials, and shaded regions indicate $95\%$ confidence intervals.}
    \label{fig:sweep}
\end{figure}

\paragraph{Sensitivity to the bonus term}
The bonus terms above omit constants and logarithmic factors. To examine the sensitivity to their scale, we multiply the bonus term of each algorithm by $c \in \{0.1, 0.3, 1, 3, 10\}$, where $c = 1$ corresponds to \Cref{fig:experiments}, and repeat $3$ simulations. The results are summarized in \Cref{fig:sweep}. \Cref{alg:main} achieves sublinear regret and constraint violation for every $c \leq 1$, its constraint violation is at most that of PD-LSCVI-UCB for every $c$, and its regret is smaller than or comparable to that of PD-LSCVI-UCB for every $c$. The regret of Algorithm 3 of \citet{chen2022learning} is negative for every $c$, as shown in the inset, and its constraint violation is about $\spn(v_g^*)$ per episode for every $c$, since the slack in its linear program does not depend on $c$.

\section{Discussions and Future Work}\label{app:discussion}

\paragraph{Extension to linear CMDPs}
We discuss whether \Cref{alg:main} extends to average-reward linear CMDPs, where $P(s'|s,a) = \langle \bphi(s,a), \bmu(s')\rangle$, $r(s,a) = \langle \bphi(s,a), \btheta_r\rangle$, and $g(s,a) = \langle \bphi(s,a), \btheta_g\rangle$ for a known feature map $\bphi: \calS\times\calA \to \bbR^d$ \citep{ghosh2023achieving,yu2026learning}. Most of the analysis does not depend on the tabular structure. \Cref{lem:J-V,lem:span} use only the Poisson equation of $\pi^*$ and the properties of $\Phi_W$. The violation analysis in the proof of \Cref{thm:main} uses only strong duality. The empirical transition kernel can be replaced by least-squares value iteration for each debt value $z$. The sum of bonus terms in \Cref{lem:sum N} can be bounded by the elliptical potential lemma \citep{abbasi2011improved}. The value function clipping can be applied for each $z$ as in \citet{yu2026learning}.

The difficulty lies in \Cref{lem:concentration}. In the tabular setting, the concentration of $\hat P_k$ holds for all bounded functions at once, and no covering argument is needed. In the linear setting, the regression error must be controlled uniformly over the class of value functions, and the bonus term scales with the square root of its log covering number \citep{jin2020provably}. In \Cref{alg:main}, the backup at $(s,a,z)$ uses $V_{k,h+1}(\cdot, \Pi_\Delta(z - g(s,a)))$. The value function at a fixed $z$ therefore depends on the regression weights of every debt value reachable from $z$ in one step. When $g$ is linear in the features, there are up to $2/\Delta$ such debt values, and the log covering number is $\bigO(d/\Delta + d^2)$ instead of $\bigO(d^2)$. Balancing the resulting bonus term $\tbigO(d^{3/2}\sqrt{T/\Delta})$ against the projection error $\bigO(\Delta T/\gamma)$ gives $\Delta \asymp T^{-1/3}$ and $\tbigO(T^{2/3})$ regret and constraint violation. This matches the bound of \citet{yu2026learning} but not the $\tbigO(\sqrt{T})$ bound of \Cref{thm:main}. Designing an algorithm that resolves the growth of the covering number due to the reachable debt values is left for future work.

\paragraph{Limitations of the Primal-Dual Approach}
In \Cref{sec:algorithm}, we discuss that extending the primal-dual approach to the average-reward setting is nontrivial. Here, we justify this by explaining the following two points: (i) the limitation of \citet{yu2026learning}---a primal-dual algorithm for learning CMDPs---and (ii) the difficulty of extending \citet{hong2025a}---an algorithm for unconstrained setting---to the constrained setting.

For (i), \citet{yu2026learning} proposes a primal-dual algorithm based on finite-horizon approximation that achieves only the suboptimal guarantee $\tbigO(T^{2/3})$. While their algorithm shares a similar structure with ours in that it also uses finite-horizon approximation, it incorporates the dual multiplier into the reward. In particular, the dual multiplier $\lambda_{k+1}$ is updated in each episode $k$, which can be simplified as $\lambda_{k+1}=\lambda_k-\eta\sum_{h=1}^H g(s_h^k,a_h^k)$, where $\eta$ is a step size. From the perspective of regret analysis, however, this update is problematic. Specifically, since $\sum_{h=1}^H g(s_h^k,a_h^k)$ can be as large as $H$, the standard dual-regret analysis incurs a term of the form $\lambda^2/(2\eta)+\eta KH^2/2$, which cannot be bounded by $\tbigO(\sqrt{T})$ under the relation $T=KH$. Consequently, their regret guarantee is limited to $\tbigO(T^{2/3})$.

For (ii), \citet{hong2025a} proposes a deviation-controlled value iteration algorithm for learning average-reward MDPs. A key step in their analysis is to control the deviation of value functions, whose analogue in our setting is $\|V_{k,h+1}-V_{k+1,h+1}\|_\infty$. However, when a dual multiplier $\lambda_k$ is introduced, the algorithm operates with respect to the episode-dependent composite reward $r_k=r+\lambda_k g$. Consequently, even when the estimated transition kernel is fixed across consecutive episodes, the deviation satisfies a recursion of the form
\begin{align*}
    \|V_{k,h+1}-V_{k+1,h+1}\|_\infty
    \approx
    \|r_{k+1}-r_k\|_\infty
    +
    \|\hat P_k(V_{k+1,h+2}-V_{k,h+2})\|_\infty.
\end{align*}
Here, the additional variation term $\|r_{k+1}-r_k\|_\infty$, which is absent in the unconstrained setting, is induced by the primal-dual extension. Since $\lambda_{k+1}\neq\lambda_k$, this creates an additional obstacle to directly extending their deviation-control argument to the constrained setting.

\subsection*{Use of Generative-AI Tools}
In this work, we used generative AI tools to obtain correct state augmentations via Huber potential, to find related work on state augmentation in reinforcement learning, and to assist with the implementation of numerical experiments.  The potential-difference reward design, the drafting of the manuscript, and theoretical developments and proofs were carried out by the authors. We have carefully reviewed all AI-assisted work, including the experimental code. We take responsibility for the final content of this work, including text, claims, proofs, and artifacts produced with the aid of generative AI.

\end{document}